\documentclass{article}

     \PassOptionsToPackage{numbers, sort}{natbib}
\pdfoutput=1

     \usepackage[preprint]{neurips_2020}

\usepackage[utf8]{inputenc} 
\usepackage[T1]{fontenc}    
\usepackage{amsfonts}       
\usepackage{nicefrac}       

\usepackage{enumitem}
\setlist[itemize]{itemsep=0.01cm, topsep=0pt}

\usepackage{microtype}
\usepackage{graphicx}
\usepackage{subcaption}
\usepackage{booktabs}

\usepackage{amsmath,amsbsy,amsgen,amscd,amssymb,amsthm,amsfonts,stmaryrd,mdframed}
\usepackage{mathtools}

\usepackage{pifont}

\usepackage{bbm}

\usepackage{url}

\usepackage[usenames,dvipsnames,svgnames]{xcolor}
\usepackage{graphicx}

\definecolor{dark-gray}{gray}{0.3}
\definecolor{dkgray}{rgb}{.4,.4,.4}
\definecolor{dkblue}{rgb}{0,0,.5}
\definecolor{medblue}{rgb}{0,0,.75}
\definecolor{rust}{rgb}{0.5,0.1,0.1}
\definecolor{darkblue}{rgb}{0,0.08,0.45}

\usepackage[colorlinks]{hyperref}
\hypersetup{urlcolor=rust}
\hypersetup{citecolor=darkblue}
\hypersetup{linkcolor=blue}

\usepackage{cleveref}

\newtheorem{theorem}{Theorem}
\newtheorem{lemma}{Lemma}
\newtheorem{corollary}{Corollary}

\newtheorem{remark}{Remark}

\renewcommand{\phi}{\varphi}

\newenvironment{sproof}{%
  \proof}{\endproof}

\DeclareMathOperator*{\argmin}{argmin}
\DeclareMathOperator{\dist}{dist}
\DeclareMathOperator{\diag}{diag}

\DeclareMathOperator{\prox}{prox}

\newcommand{\lr}[1]{\langle #1\rangle}
\newcommand{\n}[1]{\|#1 \|}
\renewcommand{\a}{\alpha}

\newcommand{\hv}{\hat v}

\newcommand{\hx}{\hat x}
\newcommand{\E}{\mathbb{E}}
\newcommand{\R}{\mathbb{R}}
\newcommand{\one}{\mathbbmss{1}}

\usepackage{xcolor,colortbl}

\newcommand{\cX}{\mathcal X}

\makeatletter
\newtheorem*{rep@theorem}{\rep@title}
\newcommand{\newreptheorem}[2]{%
        \newenvironment{rep#1}[1]{%
                \def\rep@title{#2 \ref{##1}}%
                \begin{rep@theorem}}%
                {\end{rep@theorem}}}
\makeatother

\newreptheorem{theorem}{Theorem}
\newreptheorem{lemma}{Lemma}

\usepackage{lscape}
\usepackage{makecell}

\usepackage{algorithm,algorithmic}

\crefname{lemma}{lemma}{lemmas}

\usepackage[colorinlistoftodos,prependcaption]{todonotes}

\newmdtheoremenv[
outerlinewidth=2,
roundcorner=10 pt,
leftmargin=1,
rightmargin=1,
outerlinecolor=blue!70!black,
innertopmargin=\topskip,
splittopskip=\topskip,
ntheorem=true]{assumption}{Assumption}

\title{Convergence of adaptive algorithms for weakly convex constrained optimization}

\author{%
  Ahmet Alacaoglu \\
  EPFL\\
  \texttt{ahmet.alacaoglu@epfl.ch} \\
   \And
   Yura Malitsky \\
   EPFL \\
  \texttt{yurii.malitskyi@epfl.ch} \\
   \And
   Volkan Cevher \\
   EPFL \\
    \texttt{volkan.cevher@epfl.ch} \\
}

\begin{document}

\maketitle

\begin{abstract}
We analyze the adaptive first order algorithm AMSGrad, for
solving a constrained stochastic optimization problem with a weakly
convex objective. We prove the $\mathcal{\tilde O}(t^{-1/4})$ rate of convergence for
the norm of the gradient of Moreau envelope, which is the standard
stationarity measure for this class of problems.  It matches the known
rates that adaptive algorithms enjoy for the specific case of
unconstrained smooth stochastic optimization.  Our analysis
works with mini-batch size of $1$, constant first and second order
moment parameters, and possibly unbounded optimization domains.
Finally, we illustrate the applications and extensions of our results to specific
problems and algorithms.
\end{abstract}

\section{Introduction}
Adaptive first order methods have become a mainstay of neural network
training in recent years. Most of these methods build on the AdaGrad
framework~\cite{duchi2011adaptive}, which is a modification of online
gradient descent by incorporating the sum of the squared gradients
in the step size rule.  Based on the practical shortcomings of
AdaGrad for training neural networks,
RMSprop~\cite{tieleman2012lecture} and Adam~\cite{kingma2015adam}
proposed to use exponential moving averages for gradients and squared gradients with parameters $\beta_1$ and $\beta_2$, respectively. These
methods have seen a huge practical success.

The recent work~\cite{reddi2018convergence} identified a technical issue that affects Adam and RMSprop and proposed a new Adam-variant called AMSGrad that does not suffer from the same problem.
Theoretical properties of AMSGrad, AdaGrad and their variants for nonconvex optimization problems are studied in a number of recent papers~\cite{chen2019convergence,chen2020closing,zou2019sufficient,ward2019adagrad,li2019convergence,defossez2020convergence}.
These works focus on unconstrained smooth stochastic
optimization, where the standard analysis framework of the stochastic gradient descent
(SGD)~\cite{ghadimi2013stochastic} can be used.
Convergence of adaptive methods for the more general setting of
constrained and/or nonsmooth stochastic nonconvex optimization has remained open,
while these settings have broad practical applications~\cite{vieillard2019connections,marquez2017imposing,madry2018towards,ilyas2018black,davis2019stochastic,drusvyatskiy2019efficiency}.

In this work, we take a step towards this direction and establish the convergence of AMSGrad for solving the problem
\begin{equation}\label{eq: prob}
\min_{x\in\mathcal{X}} \left\{ f(x) = \mathbb{E}_\xi\left[ f(x; \xi)\right]\right\},
\end{equation}
where $f\colon \R^d\to \R$ is $\rho$-weakly convex,
$\mathcal{X}\subset \R^d$ is a closed convex set, and $\xi$ is a random variable following a fixed unknown distribution.
This template captures the setting of previous analyses when $f$ is $L$-smooth, as this implies $L$-weak convexity, and $\mathcal{X}=\mathbb{R}^d$.
On the other hand, there exist many applications when $\mathcal{X}\neq\mathbb{R}^d$~\cite{vieillard2019connections,marquez2017imposing,madry2018towards,ilyas2018black} or when $f$ is not $L$-smooth~\citep[Section 2.1]{davis2019stochastic},\cite{duchi2018stochastic,drusvyatskiy2019efficiency}.

It is well known that constrained stochastic optimization with
nonconvex functions presents challenges not met in the convex setting~\cite{ghadimi2016mini,chen2019zo}.
In particular, until the recent work of~\citet{davis2019stochastic},
even for SGD, increasing mini-batch sizes were required for
convergence in constrained optimization.
To study the behavior of AMSGrad for solving~\eqref{eq: prob}, we build on the analysis framework of~\cite{davis2019stochastic}.

\textbf{Contributions.} We show that AMSGrad achieves
$\mathcal{O}(\log(T)/\sqrt{T})$ rate for near-stationarity, see~\eqref{eq: mor1}, for solving~\eqref{eq: prob}.  Key
specifications for this result are the following:
%
\begin{itemize}
\item We can use a mini-batch size of $1$.
\item We can use constant moment parameters $\beta_1, \beta_2$ which are used in practice~\cite{kingma2015adam,reddi2018convergence,chen2019convergence,alacaoglu2020new}.
\item We do not assume boundedness of the domain $\mathcal{X}$.
\end{itemize}
We present particular cases of our results for constrained optimization with $L$-smooth objectives and for a variant of RMSprop.
We also extend our analysis for the scalar version of AdaGrad with first order momentum.
For easy reference, we compare
our results with state-of-the-art in \Cref{table:1}.
\subsection{Related work}
Adaptive algorithms based on AdaGrad~\cite{duchi2011adaptive} and
Adam~\cite{kingma2015adam,reddi2018convergence,alacaoglu2020new} are
classically analyzed in online optimization framework with convex
objective functions.  Recent works studied the behavior of these
methods for nonconvex
optimization~\cite{li2019convergence,ward2019adagrad,zou2018weighted,chen2020closing,chen2019convergence,zou2019sufficient,defossez2020convergence,barakat2019convergence}.
The common characteristic of these results is that they are based
on the well established proof templates of
SGD~\cite{ghadimi2013stochastic} that only works in the simplest
case of unconstrained smooth stochastic minimization.
Moreover, as mentioned in~\cite{alacaoglu2020new}, unconstrained optimization  makes it easier to use a constant $\beta_1$ parameter in Adam-type methods.
In particular, many results for constrained optimization require a
fast diminishing schedule for $\beta_1$ parameter, while a constant parameter is used in practice~\cite{kingma2015adam,reddi2018convergence,chen2019zo}.
\begin{table}[t]
\setlength{\belowcaptionskip}{-1pt}
\captionsetup{font=footnotesize}
\centering
\begin{tabular}{  c c c c c c }
& $f$ & $\mathcal{X}$ & $\beta_1$ & $\beta_2 $ & mini-batch size \\
\hline
\addlinespace[0.1cm]
\cite{chen2019convergence,chen2020closing,defossez2020convergence}  & $L$-smooth & $\mathbb{R}^d$  & const. & const. & 1 \\
\addlinespace[0.1cm]
\cite{chen2019zo}  & $L$-smooth & closed convex & 0 & const. & $\sim \sqrt{t}$\\
\addlinespace[0.1cm]
\cite{davis2019stochastic} & $\rho$-weakly convex & closed convex & 0 & n/a$^{*}$ & 1\\
\addlinespace[0.1cm]
\cite{mai2020convergence}$^{**}$ & $\rho$-weakly convex & closed convex & const. & n/a$^{*}$ & 1\\
\addlinespace[0.1cm]
This work& $\rho$-weakly convex & closed convex & const. & const. &1 \\
\hline
\end{tabular}
\vspace{1mm}
\caption{$^{*}$These algorithms do not include adaptive step sizes involving observed stochastic (sub)gradients. In the notation of~\Cref{alg:alg1}, these methods corresponds to $\hat v_t =1$.
$^{**}$The algorithm of~\cite{mai2020convergence} is slightly different than how we describe, in how it sets the vector $m_t$. This simplification does not affect the message of the table.}
\label{table:1}
\end{table}

The specific case of~\eqref{eq: prob} with $L$-smooth $f$ is studied
by~\citet{chen2019zo}, where the authors proposed a zeroth order
variant of AMSGrad.  This result applies for the specific case of
$\beta_1 = 0$ which corresponds to a variant of
RMSprop~\cite{tieleman2012lecture,reddi2018convergence}.  More
importantly, since its analysis follows the one
of~\citet{ghadimi2016mini}, increasing mini-batch sizes are
required~\cite[Theorem 2]{chen2019zo}.

As also mentioned in~\cite{chen2019zo,davis2019stochastic}, analysis of SGD for constrained problems introduces specific difficulties that are not observed in the convex case.
Due to this, classical works analyzing SGD for nonconvex constrained
optimization used large mini-batches to ensure
convergence~\cite{ghadimi2016mini}.
Showing convergence for SGD for constrained optimization with a single
sample had been an open question until~\citet{davis2019stochastic}
gave a positive answer in the framework of weakly convex stochastic optimization,
which includes constrained smooth stochastic optimization as a special case.

Weakly convex optimization is well studied with SGD based
methods~\cite{duchi2018stochastic,davis2019proximally,davis2019stochastic}.
A recent work
by~\citet{mai2020convergence}, considers momentum SGD for
problem~\eqref{eq: prob}. However, this algorithm (i) does not use
momentum with $\beta_2$ and (ii) uses a scalar stepsize with
$\hat v_t = 1$ in the notation of~\Cref{alg:alg1} and
$\alpha_t = \alpha/\sqrt{T}$. These make the algorithm less practical,
while simpler for analysis.

Another promising direction of research concerns nonsmooth nonconvex
problems under more general assumptions. For instance,
\textit{tameness} and \textit{Hadamard semi-differentiability} are
used in~\cite{davis2020stochastic} and~\cite{zhang2020complexity},
respectively, where convergence guarantees are established for
SGD-based methods.  Because of the generality of the problem class
in these works, the algorithms studied there are simpler than the Adam-type
algorithms considered in this paper, and the stationarity measures are less standard~\cite{zhang2020complexity}.

\begin{algorithm}
\begin{algorithmic}
    \STATE {\bfseries Input:} $x_1 \in \mathcal{X}$,
    $\a_t=\frac{\alpha}{\sqrt{t}}$, for $t \geq 1$, $\a>0$, $\beta_1 < 1$, $\beta_2 < 1$,\\
    $m_0=v_0=0$, $\hv_0=\delta \one$, $1\geq \delta > 0$.
    \FOR{$t = 1,2\ldots T$}
        \STATE $g_t \in  \partial f(x_t, \xi_t)$
         \STATE $m_{t}= \beta_{1}m_{t-1} + (1-\beta_{1})g_t$
        \STATE $v_t= \beta_{2} v_{t-1} + (1-\beta_{2}) g_t^2$
         \STATE ${\hat{v}_t= \max(\hat{v}_{t-1}, v_t) }$
         \STATE $x_{t+1}= P^{{\hat{v}_t}^{1/2}}_{\mathcal{X}} (x_t - \alpha_t {\hat{v}_t}^{-1/2} m_t )$
    \ENDFOR
        \STATE {\bfseries Output:} $x_{t^\ast}$, where $t^\ast$ is selected uniformly at random from $\{ 1, \dots, T \}$.
\end{algorithmic}
\caption{AMSGrad~\cite{reddi2018convergence} }
\label{alg:alg1}
\end{algorithm}
\subsection{Notation}
We adopt the convention of using the standard operations $ab$, $a^2$, $a/b$, $a^{1/2}$, $1/a$, $\max\{a,b\}$ as element-wise, given two vectors $a, b\in\mathbb{R}^d$.
To denote $i^{\text{th}}$ element of the vector $a_t\in\mathbb{R}^d$, we use the notation $a_{t,i}$.
All-ones vector is denoted as $\one$.
Given a vector $a\in\mathbb{R}^{d}$, we define the matrix
$\diag(a)\in\mathbb{R}^{d\times d}$ as the matrix of all zeros, except
the diagonal, where the elements of $a$ are inserted.
For any set $\cX$, indicator function $I_{\cX}$ is given by
$I_{\cX}(x)=0$ if $x\in \cX$; and $I_{\cX}(x)=+\infty$
otherwise.

Given the elements $v_i > 0$,  $i =1,\ldots,d$, we
define a weighted norm $\|x\|_v^2 \coloneqq \langle x, \diag (v) x \rangle$.
The weighted projection operator onto $\mathcal{X}$ is defined as
\begin{equation}
 P^v_{\mathcal{X}}(x) = \argmin_{y\in\mathcal{X}} \| y- x \|^2_{v}.\label{eq: w_proj_def}
\end{equation}
A standard property of the weighted projection is that $\forall x, y \in \mathbb{R}^d$, $P_{\mathcal{X}}^v$ is
nonexpansive:
\begin{equation}
\| P^v_{\mathcal{X}} (y) - P^v_{\mathcal{X}} (x) \|_v \leq \| y - x \|_v.\label{eq: nonexp}
\end{equation}
Due to nonconvexity, we cannot use standard definition of subgradients to form a global under-estimator.
\emph{Regular subdifferential}, denoted as $\partial f$, for nonconvex functions~\citep[Ch. 8]{rockafellar2009variational} is defined as the set of vectors $q\in\mathbb{R}^d$ such that, $\forall x, y\in\mathbb{R}^d$, $q\in\partial f(x)$ if
\begin{equation}
f(y) \geq f(x) + \langle y-x, q \rangle + o(\| y -x \|), \text{~~~~ as } y \to x.
\end{equation}
When $f$ is convex, this reduces to standard definition of a
subdifferential and when $f$ is differentiable, this set coincides
with $\{\nabla f(x)\}$.

We say that $f$ is \emph{$\rho$-weakly convex} w.r.t.\ $\| \cdot \|$, if $f(x) + \frac{\rho}{2} \| x \|^2$ is convex.
An equivalent representation for weakly convex functions is that, $\forall x,y \in\mathbb{R}^d$, where $q\in\partial f(x)$~\citep[Lemma 2.1]{davis2019stochastic},
\begin{equation}
f(y) \geq f(x) + \langle y-x, q \rangle - \frac{\rho}{2} \| y-x \|^2.
\end{equation}
Moreover, we say $f$ is \emph{$L$-smooth},  if it holds that, $\forall x,y\in\mathbb{R}^d$
\begin{equation}
\| \nabla f(x) - \nabla f(y) \| \leq L \| x- y \|.
\end{equation}
Given random iterates $x_1, \dots, x_t$, we denote the filtration generated by these realizations as $\mathcal{F}_t=\sigma(x_1, \dots, x_t)$, and the corresponding conditional expectation as $\mathbb{E}_t [\cdot] = \mathbb{E}[\cdot | \mathcal{F}_t]$.
By the law of total expectation, it directly follows that $\mathbb{E}\left[ \mathbb{E}_t[\cdot]\right]=\mathbb{E}[\cdot]$.

We now present the assumptions of our analysis.
\begin{assumption}\label{as: as1}~\\
$\bullet$ $f\colon\mathbb{R}^d\to\mathbb{R}$ is $\rho$-weakly convex with respect to norm $\| \cdot \|$. \\[1mm]
$\bullet$ The set $\mathcal{X}\subset \mathbb{R}^d$ is convex and closed. \\[1mm]
$\bullet$ There exists $g_t \in \partial f(x_t, \xi_t)$ such that $\| g_t\|_{\infty} \leq G, \forall t$.\\[1mm]
$\bullet$ $f$ is lower bounded: $f^\star \leq f(x), \forall x\in\mathcal{X}$.
\end{assumption}

\begin{remark}\label{rem: rem1}
We note that when $f$ is $ \rho$-weakly convex w.r.t.\ $\| \cdot \|$, then it is $\frac{ \rho}{\sqrt{\delta}}$-weakly convex w.r.t.\ $ \|
\cdot\|_{\hat v_t^{1/2}}$, $\forall t$, since $\hat v_{t, i} \geq \delta$~(see~\Cref{alg:alg1}).
We denote
$\hat \rho = \frac{ \rho}{\sqrt{\delta}}$.
\end{remark}
It is easy to verify this remark by noticing that $x\mapsto f(x) + \frac{\rho}{2} \| x \|^2$ is convex and $\frac{\hat\rho}{2} \| x \|^2_{\hat v_t^{1/2}} \geq \frac{\rho}{2} \| x \|^2$.

A few remarks are in order for~\Cref{as: as1}.
First, we do not require boundedness of the domain $\mathcal{X}$.
Second, weak convexity assumption is weaker than smoothness assumption on $f$ and the assumption of bounded gradients is standard~\cite{chen2020closing,chen2019convergence,defossez2020convergence}.
In principle, it is possible to relax the bounded gradient assumption to the weaker requirement $\mathbb{E} \| g_t \|^2 \leq G$ as in~\cite[Remark 6.~(ii)]{zou2019sufficient} with a slightly worse and complicated convergence rate.
Thus, for clarity, we stick with~\Cref{as: as1}.

\section{Algorithm and preliminaries}
We analyze the algorithm AMSGrad  proposed in~\cite{reddi2018convergence}. On top of Adam~\cite{kingma2015adam}, it includes a step to ensure monotonicity of the exponential average of squared gradients. It is standard in stochastic nonconvex optimization to output a randomly selected iterate~\cite{davis2019stochastic,ghadimi2013stochastic,ghadimi2016mini}, which we also adopt.

We next define the composite objective
\begin{equation}
\phi(x) = f(x) + I_\mathcal{X}(x).\notag
\end{equation}
For nonsmooth problems, the standard stationarity measures such as the norm
of subgradients are no longer applicable,
see~\cite{davis2019stochastic,mai2020convergence} and \citep[Section
4]{drusvyatskiy2019efficiency}.  This motivates the following
definitions that, as we show below, relate to a relaxed form of
stationarity.  Based on $\phi$ and a parameter $\bar\rho > 0$, we
define the proximal point of $x_t$ and the Moreau envelope
\begin{align}
\hat{x}_t = \prox^{\hat{v}_t^{1/2}}_{\phi/\bar{\rho}}(x_t) &= \argmin_y\left\{ \phi(y) + \frac{\bar{\rho}}{2} \| y-x_t \|^2_{\hat v_t^{1/2}}\right\}  \label{eq: def_hatxt}\\
\phi^t_{1/\bar{\rho}}(x_t) &= \min_y \left\{\phi(y) + \frac{\bar{\rho}}{2}\| y-x_t \|^2_{\hat v_t^{1/2}}\right\}.\notag
\end{align}
We compare the definitions with that of~\citet{davis2019stochastic}.
Due to the use of variable metric $\hat v_t$ in adaptive methods, we have a time dependent Moreau envelope, where the corresponding vector $\hat v_t$ is used for defining the norm.
Important considerations for these quantities are the uniqueness of $\hat x_t$ and the smoothness of $\phi^t_{1/\bar\rho}$.
As we shall see now, choice of $\bar\rho$ is critical for ensuring these.

In light of~\Cref{rem: rem1}, selecting $\bar\rho > \hat \rho= \frac{\rho}{\sqrt\delta}$, and by using similar arguments as in~\cite[Lemma 2.2]{davis2019stochastic}, it follows that $\hat x_t$ is unique and $\phi^t_{1/\bar\rho}$ is smooth with the gradient
\begin{equation}
\nabla \phi^t_{1/\bar\rho}(x_t) = \bar\rho \hat v_t^{1/2} (x_t - \hat x_t).\notag
\end{equation}

\textbf{Near stationarity.}
Near-stationarity conditions follow from the optimality condition of $\hat x_t$: $0\in\partial \phi(\hat x_t) + \bar\rho \hat v_t^{1/2}(\hat x_t - x_t)$, where we have used $\hat v_{t,i} \leq G^2$:
\begin{equation}\label{eq: mor1}
\begin{cases}
\| x_t - \hat x_t \|^2_{\hat v_t^{1/2}} = \frac{1}{\bar \rho^2} \| \nabla \phi^t_{1/\bar\rho}(x_t) \|^2_{\hat v_t^{-1/2}} \\
\dist^2(0, \partial \phi (\hat x_t)) \leq  {G} \| \nabla \phi^t_{1/\bar\rho}(x_t) \|^2_{\hat v_t^{-1/2}} \\
\phi(\hat x_t) \leq \phi(x_t).
\end{cases}
\end{equation}
Consistent with previous literature~\cite{davis2019stochastic,mai2020convergence}, we will state the convergence guarantees in terms of the norm of the gradient of Moreau envelope.
Given~\eqref{eq: mor1}, this means that the iterate $x_t$ is close to its proximal point $\hat x_t$ and $\hat x_t$ is an approximate stationary point.

\section{Convergence analysis}
\subsection{Preliminary results}
We start with a result showing that under~\Cref{as: as1}, the quantity $\| x_t - \hat x_t \|$ from~\eqref{eq: mor1} stays bounded.
This is the main reason why we do not need to assume boundedness of $\mathcal{X}$.
The proof of this lemma given in~\Cref{sec: appendixmain} combines the definition of $\hat x_t$ with weak convexity to reach the result.
\begin{lemma}\label{lem: hatx_bd}
Let~\Cref{as: as1} hold. Let $\bar\rho > \hat\rho$, and $\hat v_t \geq \delta > 0$ (see~\Cref{alg:alg1}).
It follows that
\begin{equation}
\| x_t - \hat x_t \|^2 \leq \hat D^2 \coloneqq \frac{4 dG^2}{{\delta}(\bar\rho-\hat\rho)^2}.\notag
\end{equation}
\end{lemma}
A key challenge in the analysis of adaptive algorithms is the
dependence of $\hat v_t$ and $g_t$ that couples $\hat x_t$ and $g_t$ (see~\eqref{eq: def_hatxt}),
preventing taking expectation of $\langle x_t - \hat x_t, g_t \rangle$
that we use for obtaining the stationarity measure in the proof. Since this was not the case in~\cite{davis2019stochastic}, we need a more refined analysis.
\begin{lemma}\label{lem: lhs}
Let~\Cref{as: as1} hold.
Let $q_t = \mathbb{E}_t [g_t] \in\partial f(x_t)$, then it follows that
\begin{align}
\alpha_t \mathbb{E}_t   \langle x_t - \hat x_t, g_t \rangle \geq \alpha_t(\bar\rho - \hat\rho)&\mathbb{E}_t  \| x_t - \hat x_t \|^2_{\hat v_t^{1/2}} -(\alpha_{t-1}-\alpha_t) \sqrt{d}\hat D G - \frac{\bar\rho-\hat\rho}{4\bar\rho} \mathbb{E}_t\| \hat x_t - \hat x_{t-1} \|^2_{\hat v_{t-1}^{1/2}} \notag \\
- \frac{\alpha_{t-1}}{2} &\mathbb{E}_t \| m_{t-1} \|^2_{\hat v_{t-1}^{-1/2}} - \left( \frac{1}{2} + \frac{\bar\rho}{\bar\rho-\hat\rho} \right)\frac{\alpha_{t-1}^2}{\sqrt{\delta}} \mathbb{E}_t \| g_t \|^2.\notag
\end{align}
\end{lemma}
We review the terms in this bound to gain some intuition.
The first term in the RHS is the stationarity measure (see~\eqref{eq: mor1}), second term will sum to a constant, fourth and fifth terms will  sum to $\log(T)$ by~\Cref{lem: sum_mt}.
Handling the third term in RHS is not as obvious, but we can show that we can cancel it using the contribution from another part of the analysis that we detail in the full proof (see~\eqref{eq: neg_term}).

As alluded earlier, one critical issue for Adam-type algorithms is to
obtain results with constant $\beta_1$ parameter. A recent paper~\cite{alacaoglu2020new} studied
this problem for constrained convex problems. The following lemma from~\cite{alacaoglu2020new} also plays an
important role in our analysis.
\begin{lemma}{\citep[Lemma 1]{alacaoglu2020new}}\label{lem: lem1}
Let $m_t = \beta_1 m_{t-1} + (1-\beta_1) g_t$.
Then for any vectors $A_{t-1}$, $A_t$, we have
\begin{align}
\langle A_t, g_t \rangle = \frac{1}{1-\beta_1}\left( \langle A_t, m_t \rangle - \langle A_{t-1}, m_{t-1} \rangle \right) + \langle A_{t-1}, m_{t-1} \rangle + \frac{\beta_1}{1-\beta_1}\langle A_{t-1}-A_t, m_{t-1}\rangle.\notag
\end{align}
\end{lemma}
This lemma derives a decomposition for handling $\beta_1$ parameter in the beginning of the analysis.
As explained in~\citep[Section 3.1]{alacaoglu2020new}, using a decomposition for $m_t$ later in the analysis results in a requirement of decreasing $\beta_1$, especially for constrained problems, which we would like to avoid.

Next lemma is a standard estimation used for the analysis of
Adam-based methods, dating back to~\cite{kingma2015adam}.
For easy reference we point out to~\cite{alacaoglu2020new} where this bound is included as a separate lemma with tighter estimations than previous works, due to using a constant $\beta_1$.
It bounds the sum of the norms of first moment vectors multiplied by
the step size sequence.
\begin{lemma}\label{lem: sum_mt}
Let $\beta_1 < 1$, $\beta_2 < 1$, $\gamma = \frac{\beta_1^2}{\beta_2} < 1$, then it holds that
\begin{equation}
\sum_{t=1}^T \alpha_t^2 \| m_t \|^2_{\hat v_t^{-1/2}} \leq \frac{1-\beta_1}{\sqrt{(1-\beta_2)(1-\gamma)}} dG(1+\log T).\notag
\end{equation}
\end{lemma}
\subsection{Main result}
Equipped with the preliminary results from the previous section, we
proceed to our main theorem that shows that the norm of the gradient of Moreau
envelope converges to $0$ at the claimed rate, resulting in near-stationarity of $ x_{t^\ast}$, as in~\eqref{eq: mor1}.

\begin{theorem}\label{th: th1}
Let~\Cref{as: as1} hold. Let $\beta_1 < 1$, $\beta_2 < 1$, $\gamma =
\frac{\beta_1^2}{\beta_2} < 1$, $\bar\rho=2\hat\rho$.
Then, for iterate $x_{t^\ast}$ generated by~\Cref{alg:alg1}, it follows that
\begin{align}
\mathbb{E} \| \nabla \phi^{t^\ast}_{1/\bar\rho}&(x_{t^\ast})\|^2_{\hat v_{t^\ast}^{-1/2}} \leq\frac{2}{\alpha \sqrt{T} } \bigg[ C_1 + (1+\log T)C_2 + C_3 \bigg],\notag
\end{align}
where
$C_1 =\frac{4\rho\beta_1\alpha}{\sqrt{\delta}(1-\beta_1)}  \sqrt{d} \hat D G + \phi^1_{1/\bar\rho}(x_1) - f^\star$,\\
$C_2 = \frac{5\rho}{{\delta}}dG^2
+\frac{2\rho}{\sqrt{\delta}}\left(1 + \frac{G}{\sqrt{\delta}} + \frac{\beta_1}{1-\beta_1} + \frac{2\beta_1^2}{(1-\beta_1)^2} \right) \frac{1-\beta_1}{\sqrt{(1-\beta_2)(1-\gamma)}} dG$,\\
$C_3 = \frac{2\rho}{\sqrt{\delta}} \hat D^2 \sum_{i=1}^d \mathbb{E}\hat v_{T+1, i}^{1/2}$, and $\hat D \coloneqq \frac{2 \sqrt{d}G}{\rho}$.
\end{theorem}
We delay the discussion about the result to~\Cref{sec: discuss} and  continue with the proof sketch of the theorem, which is a careful combination of the preliminary results mentioned in the previous section.
The sketch includes the necessary bounds, but omits the tedious estimations required in some steps.
The full proof with the details is given in~\Cref{sec: appendixmain}.
\begin{sproof}
We sum the result of Lemma~\ref{lem: lem1} and use $A_1 = A_0$. with $m_0 =0$. We note that we have $A_t = \bar\rho\alpha_t(x_t - \hat{x}_t)$, for $t\geq 1$.
\begin{align}
\sum_{t=1}^T \langle A_t, g_t \rangle = \frac{\beta_1}{1-\beta_1} \langle A_T, m_T \rangle + \sum_{t=1}^T \langle A_t, m_t \rangle + \frac{\beta_1}{1-\beta_1}\sum_{t=1}^{T-1} \langle A_t - A_{t+1}, m_t \rangle.\label{eq: key1_sketch}
\end{align}
After plugging in the value of $A_t$,~\eqref{eq: key1_sketch} becomes
\begin{align}
\sum_{t=1}^T \bar\rho\alpha_t \langle x_t - \hat{x}_t, g_t \rangle &\leq \frac{\beta_1\bar\rho\alpha_T}{1-\beta_1} \langle x_T - \hat{x}_T, m_T \rangle + \sum_{t=1}^T \bar\rho\alpha_t \langle x_t- \hat{x}_t, m_t \rangle \notag \\
&\qquad + \frac{\beta_1 \bar\rho}{1-\beta_1} \sum_{t=1}^{T-1} \langle \alpha_t(x_t - \hat{x}_t) - \alpha_{t+1}(x_{t+1} - \hat{x}_{t+1}), m_t \rangle.\label{eq: mainmain1_sketch}
\end{align}
LHS of this bound is suitable for applying~\Cref{lem: lhs} to obtain the stationarity measure.
We have to estimate the three terms on the RHS.
It is easy to bound the first term  using Cauchy-Schwarz inequality and ~\Cref{lem: hatx_bd}.
Other two terms require longer estimations which we sketch below.

$\bullet$ Bound for $\frac{\beta_1 \bar\rho}{1-\beta_1} \sum_{t=1}^{T-1} \langle \alpha_t(x_t - \hat{x}_t) - \alpha_{t+1}(x_{t+1} - \hat{x}_{t+1}), m_t \rangle$ in~\eqref{eq: mainmain1_sketch}.

Decomposing this term gives
\begin{align}
\langle \alpha_t (x_t - \hat{x}_t) - \alpha_{t+1}(x_{t+1} - \hat{x}_{t+1}), m_t \rangle = (\alpha_{t}-\alpha_{t+1}) &\langle x_{t+1} - \hat{x}_{t+1}, m_t \rangle + \alpha_t \langle x_t - x_{t+1}, m_t \rangle \notag \\
 + \alpha_t &\langle \hat x_{t+1} - \hat x_t, m_t\rangle. \notag
\end{align}
For the first term, we use that $\alpha_t \geq
\alpha_{t+1}$ and Cauchy-Schwarz inequality
\begin{align}
\sum_{t=1}^{T-1} (\alpha_{t}-\alpha_{t+1}) \langle x_{t+1} - \hat{x}_{t+1}, m_t \rangle &\leq  \sum_{t=1}^{T-1} (\alpha_t - \alpha_{t+1}) \hat D\sqrt{d}G \leq \alpha_1 \hat D\sqrt{d}G.\notag
\end{align}
For the second term we deduce by Cauchy-Schwarz inequality and nonexpansiveness of the projection
\begin{align}
\alpha_t \langle x_t - x_{t+1}, m_t \rangle &\leq \alpha_t^2 \| m_t \|^2_{\hat v_t^{-1/2}}. \notag
\end{align}
For the third term, we use Young's inequality to obtain the bound
\begin{multline}
\sum_{t=1}^{T-1} \frac{\beta_1 \bar\rho}{1-\beta_1} \langle \alpha_t(x_t - \hat x_t) - \alpha_{t+1}(x_{t+1} - \hat x_{t+1}), m_t \rangle \leq \frac{\beta_1 \bar\rho}{1-\beta_1} \alpha_1  \hat D\sqrt{d}G + \sum_{t=1}^T\frac{\beta_1\bar\rho \alpha_t^2}{1-\beta_1} \| m_t \|^2_{\hat v_t^{-1/2}} \\
+ \sum_{t=1}^T\frac{\bar\rho-\hat\rho}{4} \| \hat x_{t+1} - \hat x_t \|^2_{\hat v_t^{1/2}}
+ \frac{\bar\rho^2}{(\bar\rho-\hat\rho)}\frac{\beta_1^2}{(1-\beta_1)^2} \sum_{t=1}^T\alpha_t^2 \| m_t \|^2_{\hat v_t^{-1/2}},\label{eq: second_main_term_sketch}
\end{multline}
$\bullet$ Bound for $\sum_{t=1}^T \bar\rho\alpha_t \langle x_t- \hat{x}_t, m_t \rangle$ in~\eqref{eq: mainmain1_sketch}.

We  proceed similar to~\cite{davis2019stochastic}, with a tighter estimation (resulting in the negative term on RHS) to obtain
\begin{align}
\phi^{t+1}_{1/\bar \rho}(x_{t+1}) \leq \phi^t_{1/\bar \rho}(x_t) &+
  \bar \rho \alpha_t \langle \hat{x}_t - x_t, m_t \rangle + \frac{\bar
  \rho}{2} \alpha_t^2 \| m_t \|^2_{\hat v_t^{-1/2}} \notag \\&+ \frac{\bar\rho}{2}\| \hat{x}_t - x_{t+1} \|^2_{\hat{v}_{t+1}^{1/2} - \hat v_t^{1/2}}
- \frac{\bar\rho-\hat\rho}{2} \| \hat x_t - \hat x_{t+1} \|^2_{\hat v_{t+1}^{1/2}}.\label{eq: neg_term}
\end{align}
Then we  manipulate the fourth term on RHS with standard $\| a-b\|^2 \leq 2\|a\|^2 + 2\| b\|^2$, and~\Cref{lem: hatx_bd},
\begin{align}
\frac{\bar\rho}{2}\| \hat{x}_t - x_{t+1} \|^2_{\hat{v}_{t+1}^{1/2} - \hat v_t^{1/2}} &\leq \bar\rho\| \hat{x}_t - x_{t} \|^2_{\hat{v}_{t+1}^{1/2} - \hat v_t^{1/2}} + \frac{{G}\bar\rho}{\sqrt{\delta}}\| {x}_t - x_{t+1} \|^2_{\hat{v}_{t}^{1/2}}\notag \\
&\leq \bar\rho \hat D^2 \sum_{i=1}^d (\hat v_{t+1, i}^{1/2} - \hat v_{t, i}^{1/2}) + \frac{{G}\bar\rho}{\sqrt{\delta}}\alpha_t^2 \| m_t \|^2_{\hat{v}_{t}^{-1/2}}.\label{eq: bdd_issue}
\end{align}
We use this estimation in~\eqref{eq: neg_term} and sum to get
\begin{align}
\bar \rho \alpha_t \sum_{t=1}^T \langle x_t - \hat{x}_t, m_t \rangle &\leq \phi^1_{1/\bar\rho}(x_1) - \phi^{T+1}_{1/\bar\rho}(x_{T+1}) + \sum_{t=1}^T \left(\frac{1}{2} + \frac{{G}}{\sqrt{\delta}} \right) \bar\rho\alpha_t^2\| m_t \|^2_{\hat v_t^{-1/2}} \notag \\
&+ \bar\rho \hat D^2\sum_{i=1}^d \hat v_{T+1, i}^{1/2} - \sum_{t=1}^T \frac{\bar\rho-\hat\rho}{2} \| \hat x_t - \hat x_{t+1} \|^2_{\hat v_{t+1}^{1/2}}.\label{eq: third_main_term_sketch}
\end{align}
We collect~\eqref{eq: second_main_term_sketch} and~\eqref{eq: third_main_term_sketch} in~\eqref{eq: mainmain1_sketch}.
Finally, we have to obtain the stationarity criterion on the LHS of~\eqref{eq: mainmain1_sketch} by taking conditional expectation. This is not immediate due to coupling of $\hat x_t$, $\hat v_t$, and $g_t$.
We use~\Cref{lem: lhs} to handle this issue and the negative term in~\eqref{eq: third_main_term_sketch} is utilized to cancel the third term in the RHS of the result of~\Cref{lem: lhs}.
Then, we use~\eqref{eq: mor1}, plug in~\Cref{lem: sum_mt} and $\bar\rho=2\hat\rho$ to conclude.
\end{sproof}
\subsection{Discussion}\label{sec: discuss}
In the context of near-stationarity~\eqref{eq: mor1},~\Cref{th: th1} states that to have $x_{t^\ast}$ in~\Cref{alg:alg1} such that $\| \nabla \phi^{t^\ast}_{1/\bar\rho}(x_{t^\ast}) \|_{\hat v_{t^\ast}^{-1/2}}\leq \epsilon$, we require $\tilde{\mathcal{O}}(\epsilon^4)$ iterations. This matches the known complexities for adaptive methods in unconstrained smooth stochastic optimization~\cite{alacaoglu2020new,defossez2020convergence,ward2019adagrad,zou2018weighted,chen2019convergence,chen2020closing,li2019convergence,zou2019sufficient}, and SGD-type methods in weakly convex optimization~\cite{mai2020convergence,davis2019stochastic}.

Our first remark is about the metric of the norm used for the gradient of the Moreau envelope in~\Cref{th: th1}. We then continue to discuss the dependence of our bound w.r.t. important quantities.
\begin{remark}
By~\eqref{eq: mor1}, one has $\| \nabla \phi^{t^\ast}_{1/\bar\rho}(x_{t^\ast})\|^2_{\hat v_{t^\ast}^{-1/2}} = \bar\rho^2\| x_{t^\ast} - \hat x_{t^\ast} \|^2_{\hat v_{t^\ast}^{1/2}}$.
We note that $\| x_{t^\ast} - \hat x_{t^\ast} \|^2_{\hat v_{t^\ast}^{1/2}} \geq \sqrt{\delta} \| x_{t^\ast} - \hat x_{t^\ast} \|^2$ as $\hat v_{t, i} \geq \delta$.
It also holds that $\hat v_{t,i}\leq G^2$.
Therefore, one can convert our guarantees to $\| x_{t^\ast} - \hat x_{t^\ast} \|^2$ or $\| \nabla \phi^{t^\ast}_{1/\bar\rho}(x_{t^\ast}) \|$ by multiplying the right hand side by appropriate quantities depending on $\delta$ or $G$.
We leave the result with the metric however, as $\delta$ and $G$ are the worst case bounds.
\end{remark}
\textbf{Dependence of $\beta_1$}.
Comparing with the previous work, the scaling of our bound in terms of $\beta_1$ is $(1-\beta_1)^{-1}$ matching the state-of-the-art dependence for the unconstrained setting~\cite{alacaoglu2020new,defossez2020convergence}.

\textbf{Dependence of $d$}.  Standard dependence of $d$ in the
convergence rates for Adam-type algorithms for unconstrained case is
$d/\sqrt{T}$~\cite{alacaoglu2020new,defossez2020convergence}.\footnote{We note that in~\cite{chen2020closing} better dependence is obtained by using step sizes in the order of $\frac{1}{\sqrt{d}}$, which we do not consider, as this choice forces small step sizes.}

Even
though in~\Cref{th: th1}, the constant $C_3$ has worst case dependence
$d^2$, this is merely due to assumptions.  The main reason is
that we do not assume boundedness of the sequence $x_t$, instead we
prove the necessary result for the analysis in~\Cref{lem: hatx_bd}.  However, this result gives a bound
for $\| x_t - \hat x_t \|$, which is naturally dimension dependent.
We used this bound in~\eqref{eq: bdd_issue}, where we need to use
$\| x_t - \hat x_t \|_{\infty}$.  If we had assumed a bound for
$\| x_t - \hat x_t \|_{\infty}$, then in~\eqref{eq: bdd_issue} we
could have used it instead of~\Cref{lem: hatx_bd} to have standard
$d/\sqrt{T}$ in $C_3$. We note that boundedness assumption also would
remove a factor of $\frac{1}{\sqrt{\delta}}$ in the bound, as those
appear in the steps where we avoid boundedness assumption.

\textbf{Dependence of $\delta$}.
Our bound has a polynomial dependence of $1/\delta$ similar to~\cite{alacaoglu2020new,chen2020closing,chen2019convergence}.
In~\cite{defossez2020convergence}, a more refined technique from~\cite{ward2019adagrad} is used to have a logarithmic dependence of $1/\delta$.
This technique, used on the case of smooth unconstrained problems in these works, did not seem to apply to our setting.

\section{Applications and extensions}
\subsection{Applications}
\textbf{RMSprop}.
The counterexamples presented in~\cite{reddi2018convergence} show that RMSprop, similar to Adam might diverge in simple problems.
Setting $\beta_1=0$ in AMSGrad~\cite{reddi2018convergence} results in an algorithm similar to RMSprop, with the difference of having $\hat v_t$ as the output of the $\max$ step.
Therefore, our analysis also applies to this version of RMSprop with similar guarantees.
\begin{corollary}
Let $\beta_1 = 0$. Then, for a variant of RMSprop~\cite{reddi2018convergence}, obtained by setting $\beta_1$ in~\Cref{alg:alg1},~\Cref{th: th1} applies with $\beta_1 = 0$.
\end{corollary}
It is easy to see that $\beta_1=0$ gives a better bound in~\Cref{th: th1}.
This is in fact common for the bounds of Adam-type algorithms even in the convex case~\cite{reddi2018convergence}.
Setting nonzero momentum parameters $\beta_1$, $\beta_2$ do not predict improvement, however, in practice they are routinely observed to improve performance.

\textbf{SGD with momentum}.
When $\hat v_t = \one, \forall t$, then AMSGrad reduces to an algorithm similar to SGD with momentum.
Lack of diagonal step sizes in this case simplifies the analysis as weighted projections are not used in the algorithm.
This specific case is studied in the recent work~\cite{mai2020convergence}, with a slightly different way to set $m_t$.
Our analysis can be seen as an alternative derivation of convergence for a method similar to~\cite{mai2020convergence}.

\textbf{Constrained smooth optimization}.
A special case of~\eqref{eq: prob} is when $f$ is $L$-smooth.
In this case, the standard convergence measure is the gradient mapping~\cite{ghadimi2016mini}, which is used in~\cite{chen2019zo}
\begin{equation}
\mathcal{G}_{\lambda}(x) =  \frac{\hat v_t^{1/4}}{\lambda} \left( x -
  P^{\hat v_t^{1/2}}_\mathcal{X}(x - \lambda\hat v_t^{-1/2} \nabla
  f(x)) \right).
\end{equation}
It is instructive to observe that when $\mathcal{X}=\mathbb{R}^d$, then $\|\mathcal{G}_\lambda(x)\| = \|\nabla f(x)\|_{\hat v_t^{-1/2}}\geq \frac{1}{G} \| \nabla f(x)\|$ which is the stationarity measure for smooth unconstrained problems.
In the cases when $\mathcal{X}\neq\mathbb{R}^d$, gradient mapping is used as a standard stationarity measure~\cite{ghadimi2016mini,davis2019stochastic,mai2020convergence}.

As illustrated in~\cite{davis2019stochastic}, for the specific case of constrained smooth minimization, norm of the Moreau envelope is of the same order as the norm of the gradient mapping, therefore, the results can be converted to guarantees on gradient mapping norms.
Using similar ideas as in~\citep[Theorem 3.5]{drusvyatskiy2018error},~\cite{davis2019stochastic}, one can show that $\| \mathcal{G}_{1/\bar\rho}(x_t) \| \leq C_{\mathrm{g, m}} \| \nabla \phi^t_{1/\bar\rho}(x_t) \|_{\hat v_t^{-1/2}}$, for a constant $C_{\mathrm{g, m}}$ (see~\Cref{sec: grad_map_stuff_moreau}).

\subsection{An extension: Scalar AdaGrad with momentum}\label{subs:adagrad}
An alternative adaptive algorithm is AdaGrad~\cite{duchi2011adaptive} and its variants with first order momentum are referred to as AdamNC~\cite{reddi2018convergence} or AdaFOM~\cite{chen2019convergence}.
In unconstrained smooth stochastic optimization, it has been observed that the
same proof structure applies to AMSGrad and AdaGrad-based methods
simultaneously~\cite{chen2019convergence,defossez2020convergence}. However, in our setting, the analysis we developed for AMSGrad does not directly apply to AdaGrad-based methods.

The main reason is that $v_t$ in the case of AdaGrad does not admit a
lower bound separated from $0$, unlike AMSGrad where
$0 < \delta \leq \hat v_t$.  The uniform lower bound is necessary for
converting regular weak convexity assumption w.r.t.\ norm
$\| \cdot \|$ to the one w.r.t.\ the weighted norm
$\| \cdot \|_{v_t^{1/2}}$ in the sense of~\Cref{rem: rem1}.  Naively
assuming the existence of $\hat\rho$ in~\Cref{rem: rem1}
is not consistent, since $v_t$ is not separated from zero due to
$v_t \geq \frac{\delta}{\sqrt{t}}$ in AdaGrad, and hence, the norm
$\| \cdot \|_{ v_t^{1/2}}$ is not well-defined.

In this section, we provide partial results on this direction.  In
particular, we show that the scalar version of AdaGrad, that is used
for example
in~\cite{ward2019adagrad,li2019convergence,levy2017online,levy2018online},
along with its variant with first order moment estimation also has the
same convergence rate.  In the framework of~\Cref{alg:alg1},
\textit{scalar} (non-diagonal) version of these methods iterate as,
for $g_t \in \partial f(x_t, \xi_t)$,
\begin{equation}\label{eq: adagrad}
\begin{cases}
&m_{t}= \beta_{1}m_{t-1} + (1-\beta_{1})g_t  \\
&v_t= \frac{1}{t} (\delta+ \frac{1}{d}\sum_{j=1}^t \| g_j \|^2)  \\
&x_{t+1}= P_{\mathcal{X}} (x_t - \frac{\alpha_t}{\sqrt{v_t}} m_t ),
\end{cases}
\end{equation}
where $P_{\mathcal{X}}$ denotes a standard Euclidean projection
(without any metric).  The factor of $1/d$ in front of gradient norms
is to normalize the step size, as $\ell_2$-norm is dimension
dependent.  This factor only affects the dimension dependence of the
bound.

In this case, one does not need the time-dependent definitions for Moreau envelope and proximal point.
Thus, one can define $\hat x_t = \prox_{1/\bar\rho}(x_t)$ and $\phi_{1/\bar\rho}(x) = \min_{y\in\mathcal{X}} f(y) + \frac{\bar\rho}{2} \| y - x \|^2$, due to lack of weighted projection in the algorithm since $v_t$ is now a scalar.
The proof then is similar to~\cite{davis2019stochastic} with AdaGrad step sizes. The difficulties arising due to adaptive step sizes and existence of $\beta_1$, are handled using the results in~\Cref{lem: hatx_bd},~\Cref{lem: lem1}, and~\Cref{lem: sum_mt}.
\begin{theorem}\label{th: th2}
Let~\Cref{as: as1} hold. Then, for the method sketched in~\eqref{eq: adagrad}, with $\beta_1 < 1$, $\alpha_t =\frac{\alpha}{\sqrt{t}}$ it holds
\begin{align}
\mathbb{E} \| \nabla \phi_{1/2\rho}&(x_{t^\ast})\|^2 \leq\frac{2G}{\alpha\sqrt{T} } \bigg[ C_1 + \left(1+\log\left(\frac{TG^2}{\delta}+1\right)\right)C_2 \bigg],\notag
\end{align}
where $C_1 = \phi_{1/2\rho}(x_1) - f^\star + 2\rho\left( \frac{2\beta_1}{1-\beta_1} +1 \right)\frac{\alpha \hat D \sqrt{d}G}{\sqrt{\delta}}$,
$C_2 = 2\rho\alpha d\left( \frac{1}{2} + \frac{\beta_1}{1-\beta_1} + \frac{2\beta_1^2}{(1-\beta_1)^2} \right)$, and $\hat D = \frac{2\sqrt{d}G}{\sqrt{\rho}}$.
\end{theorem}

We leave it as an open question to derive similar results for AdaGrad-based methods with diagonal step sizes.

\newpage
\section*{Acknowledgements}
This project has received funding from the European Research Council (ERC) under the European Union's Horizon $2020$ research and innovation programme (grant agreement no $725594$ - time-data), the Swiss National Science Foundation (SNSF) under grant number $200021\_178865 / 1$, the Department of the Navy, Office of Naval Research (ONR)  under a grant number N62909-17-1-211, and the Hasler Foundation Program: Cyber Human Systems (project number 16066).

\bibliographystyle{abbrvnat}
\bibliography{lit_adap_nc}

\newpage

\appendix
\allowdisplaybreaks
\section{Proofs}\label{sec: appendixmain}

\begin{replemma}{lem: hatx_bd}
Let~\Cref{as: as1} hold. Let $\bar\rho > \hat\rho$, and $\hat v_t \geq \delta > 0$ (see~\Cref{alg:alg1}).
It follows that
\begin{equation}
\| x_t - \hat x_t \|^2 \leq \hat D^2 \coloneqq \frac{4 dG^2}{{\delta}(\bar\rho-\hat\rho)^2}.\notag
\end{equation}
\end{replemma}
\begin{proof}
By the definition of $\hat x_t$ in~\eqref{eq: def_hatxt}, it follows that
\begin{align}
\phi(\hat x_t) + \frac{\bar\rho}{2} \| x_t - \hat x_t \|^2_{\hat v_t^{1/2}} \leq \phi(x_t) + \frac{\bar\rho}{2} \| x_t - x_t \|^2_{\hat v_t^{1/2}} = \phi(x_t).\notag
\end{align}
Next, we use $\hat\rho$-weak convexity of $\phi$ with respect to norm $\| \cdot \|_{\hat v_t^{1/2}}$ from~\Cref{rem: rem1}, and the fact that $x_t, \hat x_t \in \mathcal{X}$ to get for any vector $q_t$ such that $q_t\in\partial f(x_t)$,
\begin{equation}
\phi(x_t) - \phi(\hat x_t) \leq \langle x_t - \hat x_t, q_t \rangle + \frac{\hat\rho}{2} \| x_t - \hat x_t \|^2_{\hat v_t^{1/2}}.\notag
\end{equation}
We sum two inequalities and apply  Cauchy-Schwarz inequality
\begin{equation}\label{eq:better estimate}
\frac{\bar\rho - \hat\rho}{2} \| x_t - \hat x_t \|^2_{\hat v_t^{1/2}}
\leq \lr{x_t - \hx_t, g_t} \leq \n{q_t}_{\hat v_t^{-1/2}}\n{x_t-\hx_t}_{\hat v_t^{1/2}},\notag
\end{equation}
which yields
\begin{equation}
\frac{\bar\rho - \hat\rho}{2} \| x_t - \hat x_t \|_{\hat
    v_t^{1/2}}\leq \n{q_t}_{\hat v_t^{-1/2}}.\notag
\end{equation}
As $\hv_{t,i}\geq \delta $ and for $q_t$ such that $\mathbb{E} g_t =
q_t$, $\| q_t \|^2=\n{\E g_t}^2 \leq \mathbb{E}\| g_t\|^2\leq dG^2$ by~\Cref{as: as1}, we have
\begin{equation}
\n{q_t}^2_{\hv_t^{-1/2}}\leq
  \frac{dG^2}{\sqrt{\delta}}\notag
  \end{equation}
and the final bound follows immediately.
\end{proof}

\begin{replemma}{lem: lhs}
Let~\Cref{as: as1} hold.
Let $q_t = \mathbb{E}_t [g_t] \in\partial f(x_t)$, then it follows that
\begin{align}
\alpha_t \mathbb{E}_t   \langle x_t - \hat x_t, g_t \rangle \geq \alpha_t(\bar\rho - \hat\rho)&\mathbb{E}_t  \| x_t - \hat x_t \|^2_{\hat v_t^{1/2}} -(\alpha_{t-1}-\alpha_t) \sqrt{d}\hat D G - \frac{\bar\rho-\hat\rho}{4\bar\rho} \mathbb{E}_t\| \hat x_t - \hat x_{t-1} \|^2_{\hat v_{t-1}^{1/2}} \notag \\
- \frac{\alpha_{t-1}}{2} &\mathbb{E}_t \| m_{t-1} \|^2_{\hat v_{t-1}^{-1/2}} - \left( \frac{1}{2} + \frac{\bar\rho}{\bar\rho-\hat\rho} \right)\frac{\alpha_{t-1}^2}{\sqrt{\delta}} \mathbb{E}_t \| g_t \|^2.\notag
\end{align}
\end{replemma}
\begin{proof}
We first decompose the LHS
\begin{align}
\a_t\langle x_t - \hat{x}_t, g_t \rangle &= \a_t\langle x_t - \hat{x}_t, q_t
                                       \rangle + \a_t\langle x_t -
                                       \hat{x}_t, g_t - q_t \rangle\notag \\
                                     &=\a_t\langle x_{t} - \hat{x}_t, q_t
                                       \rangle + \langle \a_t(x_t -
                                       \hat{x}_t)-\a_{t-1}(x_{t-1}-\hx_{t-1}),
                                       g_t - q_t \rangle \notag \\
                                       &+\langle \a_{t-1}(x_{t-1}-\hx_{t-1}), g_t - q_t \rangle\label{eq: lhs_lem_eq1}
\end{align}
In this bound, the last term will be $0$ after taking conditional
expectation $\E_t$ as $\hat x_{t-1}$ depends on $\hat v_{t-1}$, which,
in turn, depends only on $g_1,\dots, g_{t-1}$, thus, independent of $g_t$.

For the first term in~\eqref{eq: lhs_lem_eq1},  we recall that  $\hat x_t \in\mathcal{X}$, $x_t\in\mathcal{X}$,
$q_t\in\partial f(x_t)$.
Then we use $\hat\rho$-weak
convexity of  $f$ with respect to $\| \cdot \|_{\hat v_t^{1/2}}$,
\begin{align}
\langle x_t - \hat x_t, q_t \rangle &\geq f(x_t) - f(\hat x_t) - \frac{\hat\rho}{2} \| x_t - \hat x_t \|^2_{\hat v_t^{1/2}} \notag \\
&= \Big( f(x_t) + \frac{\bar\rho}{2} \| x_t - x_t \|^2_{\hat v_t^{1/2}} \Big) - \Big( f(\hat x_t) + \frac{\bar\rho}{2} \| x_t - \hat x_t \|^2_{\hat v_t^{1/2}} \Big) + \frac{\bar\rho - \hat\rho}{2} \| x_t - \hat x_t \|^2_{\hat v_t^{1/2}} \notag \\
&\geq (\bar\rho - \hat\rho)\|x_t-\hat x_t \|^2_{\hat v_t^{1/2}},\label{eq: dd_std_lhs_est}
\end{align}
where the last step is due to $x\mapsto f(x) + I_{\mathcal{X}}(x)
+ \frac{\bar\rho}{2} \| x-x_t\|^2_{\hat v_t^{1/2}}$ being $\bar\rho -
\hat\rho$ strongly convex w.r.t. $\| \cdot \|_{\hat v_t^{1/2}}$, with the minimizer $\hat x_t$, and $x_t, \hat x_t \in \mathcal{X}$.

Next, we need to lower bound the second term in~\eqref{eq: lhs_lem_eq1}, for which we upper bound the term given by
\begin{multline}
 \langle \alpha_{t-1}(x_{t-1} - \hat x_{t-1}) - \alpha_{t}(x_{t} - \hat x_{t}), g_t - q_t \rangle = (\alpha_{t-1} - \alpha_{t})\langle x_{t} - \hat x_{t}, g_t - q_t \rangle  \\
+\alpha_{t-1} \langle x_{t-1} - x_{t}, g_t - q_t \rangle + \alpha_{t-1} \langle \hat x_{t}-\hat x_{t-1} , g_t - q_t \rangle. \label{eq: lem_lhs_eq2}
\end{multline}

We proceed with bounding the first term in the RHS of~\eqref{eq: lem_lhs_eq2}, using $\alpha_{t}\leq\alpha_{t-1}$,
\begin{align}
\mathbb{E}_t (\alpha_{t-1} - \alpha_{t})\langle x_t - \hat x_t, g_t - q_t \rangle &\leq  (\alpha_{t-1} - \alpha_t) \mathbb{E}_t\| x_t - \hat x_t \| \| g_t - q_t \| \notag \\
&\leq (\alpha_{t-1} - \alpha_t) {\hat D}\mathbb{E}_t \| g_t - q_t \| \notag \\
&\leq (\alpha_{t-1} - \alpha_t) {\hat D} \sqrt{\mathbb{E}_t \| g_t\|^2}\notag\\
&\leq (\alpha_{t-1} - \alpha_t) {\hat D} \sqrt{d} G,\notag
\end{align}
where the second inequality follows from~\Cref{lem: hatx_bd} and third
inequality follows from Jensen's inequality and  ${\mathbb{E}_t \| g_t
  - \E_tg_t \|^2}\leq {\mathbb{E}_t \| g_t \|^2}$.

For the second term in the RHS of~\eqref{eq: lem_lhs_eq2} we use
Cauchy-Schwarz and Young's inequalities and nonexpansiveness of weighted projection to get
\begin{align}
\mathbb{E}_t \alpha_{t-1} \langle x_{t-1} - x_t, g_t - q_t \rangle &\leq \frac{1}{2} \mathbb{E}_t\| x_t - x_{t-1} \|^2_{\hat v_{t-1}^{1/2}} + \frac{\alpha_{t-1}^2}{2} \mathbb{E}_t\| g_t - q_t \|^2_{\hat v_{t-1}^{-1/2}} \notag \\
&\leq \frac{\alpha_{t-1}^2}{2} \mathbb{E}_t\| m_{t-1} \|^2_{\hat v_{t-1}^{-1/2}} + \frac{\alpha_{t-1}^2}{2\sqrt{\delta}} \mathbb{E}_t\| g_t - q_t \|^2 \notag \\
&\leq \frac{\alpha_{t-1}^2}{2} \mathbb{E}_t\| m_{t-1} \|^2_{\hat v_{t-1}^{-1/2}} + \frac{\alpha_{t-1}^2}{2\sqrt{\delta}} \mathbb{E}_t\| g_t \|^2. \notag
\end{align}
Similarly, we estimate the third term in the RHS of~\eqref{eq: lem_lhs_eq2}
\begin{align}
\mathbb{E}_t \alpha_{t-1} \langle \hat x_{t} - \hat x_{t-1}, g_t - q_t \rangle &\leq \frac{\bar\rho - \hat\rho}{4\bar\rho} \|\hat x_t -\hat x_{t-1} \|^2_{\hat v_t^{1/2}} + \frac{\alpha_{t-1}^2\bar\rho}{\bar\rho-\hat\rho} \mathbb{E}_t \|g_t - q_t \|^2_{\hat v_t^{-1/2}} \notag \\
&\leq \frac{\bar\rho - \hat\rho}{4\bar\rho} \|\hat x_t -\hat x_{t-1} \|^2_{\hat v_t^{1/2}} + \frac{\alpha_{t-1}^2\bar\rho}{(\bar\rho-\hat\rho)\sqrt{\delta}} \mathbb{E}_t \|g_t \|^2.\notag
\end{align}
Combining all the bounds gives the result.
\end{proof}

\begin{replemma}{lem: sum_mt}
Let $\beta_1 < 1$, $\beta_2 < 1$, $\gamma = \frac{\beta_1^2}{\beta_2} < 1$, then it holds that
\begin{equation}
\sum_{t=1}^T \alpha_t^2 \| m_t \|^2_{\hat v_t^{-1/2}} \leq \frac{1-\beta_1}{\sqrt{(1-\beta_2)(1-\gamma)}} dG(1+\log T).\notag
\end{equation}
\end{replemma}
\begin{proof}
We start with the result of~\citep[Lemma 3]{alacaoglu2020new}
\begin{equation}
\|m_t\|^2_{\hat v_t^{-1/2}} \leq \frac{(1-\beta_1)^2}{\sqrt{(1-\beta_2)(1-\gamma)}}\sum_{i=1}^d \sum_{j=1}^t \beta_1^{t-j} \vert g_{j, i} \vert.\notag
\end{equation}
We will proceed similar to~\citep[Lemma 4]{alacaoglu2020new} with the only change of having $\alpha_t^2$ instead of $\alpha_t$
\begin{align}
\sum_{t=1}^T \alpha_t^2 \| m_t \|^2_{\hat v_t^{-1/2}} &\leq \frac{(1-\beta_1)^2}{\sqrt{(1-\beta_2)(1-\gamma)}} \sum_{i=1}^d \sum_{t=1}^T \alpha_t^2 \sum_{j=1}^t\beta_1^{t-j}\vert g_{j,i}\vert \notag \\
&= \frac{(1-\beta_1)^2}{\sqrt{(1-\beta_2)(1-\gamma)}} \sum_{i=1}^d \sum_{j=1}^T \sum_{t=j}^T \alpha_t^2 \beta_1^{t-j}\vert g_{j,i}\vert \notag \\
&\leq \frac{1-\beta_1}{\sqrt{(1-\beta_2)(1-\gamma)}} \sum_{i=1}^d \sum_{j=1}^T  \alpha_j^2 \vert g_{j,i}\vert \notag \\
&\leq \frac{1-\beta_1}{\sqrt{(1-\beta_2)(1-\gamma)}} dG(1+\log T).\notag\qedhere
\end{align}
\end{proof}

\begin{reptheorem}{th: th1}
Let~\Cref{as: as1} hold. Let $\beta_1 < 1$, $\beta_2 < 1$, $\gamma =
\frac{\beta_1^2}{\beta_2} < 1$, $\bar\rho=2\hat\rho$.
Then, for iterate $x_{t^\ast}$ generated by~\Cref{alg:alg1}, it follows that
\begin{align}
\mathbb{E} \| \nabla \phi^{t^\ast}_{1/\bar\rho}&(x_{t^\ast})\|^2_{\hat v_{t^\ast}^{-1/2}} \leq\frac{2}{\alpha \sqrt{T} } \bigg[ C_1 + (1+\log T)C_2 + C_3 \bigg],\notag
\end{align}
where
$C_1 =\frac{4\rho\beta_1\alpha}{\sqrt{\delta}(1-\beta_1)}  \sqrt{d} \hat D G + \phi^1_{1/\bar\rho}(x_1) - f^\star$,\\
$C_2 = \frac{5\rho}{{\delta}}dG^2
+\frac{2\rho}{\sqrt{\delta}}\left(1 + \frac{G}{\sqrt{\delta}} + \frac{\beta_1}{1-\beta_1} + \frac{2\beta_1^2}{(1-\beta_1)^2} \right) \frac{1-\beta_1}{\sqrt{(1-\beta_2)(1-\gamma)}} dG$,\\
$C_3 = \bar\rho \hat D^2 \sum_{i=1}^d \mathbb{E}\hat v_{T+1, i}^{1/2}$, and $\hat D \coloneqq \frac{2 \sqrt{d}G}{\rho}$.
\end{reptheorem}

\begin{proof}
We sum the result of Lemma~\ref{lem: lem1} and use $A_1 = A_0$. with $m_0 =0$. We note that we have $A_t = \bar\rho\alpha_t(x_t - \hat{x}_t)$, for $t\geq 1$.
\begin{align}
\sum_{t=1}^T \langle A_t, g_t \rangle = \frac{\beta_1}{1-\beta_1} \langle A_T, m_T \rangle + \sum_{t=1}^T \langle A_t, m_t \rangle + \frac{\beta_1}{1-\beta_1}\sum_{t=1}^{T-1} \langle A_t - A_{t+1}, m_t \rangle.\label{eq: key1}
\end{align}
After plugging in the value of $A_t$,~\eqref{eq: key1} becomes
\begin{align}
\sum_{t=1}^T \bar\rho\alpha_t \langle x_t - \hat{x}_t, g_t \rangle &\leq \frac{\beta_1\bar\rho\alpha_T}{1-\beta_1} \langle x_T - \hat{x}_T, m_T \rangle + \sum_{t=1}^T \bar\rho\alpha_t \langle x_t- \hat{x}_t, m_t \rangle \notag \\
&\qquad + \frac{\beta_1 \bar\rho}{1-\beta_1} \sum_{t=1}^{T-1} \langle \alpha_t(x_t - \hat{x}_t) - \alpha_{t+1}(x_{t+1} - \hat{x}_{t+1}), m_t \rangle.\label{eq: mainmain1}
\end{align}
LHS of this bound is suitable for applying~\Cref{lem: lhs} to obtain the stationarity measure.
We have to estimate the three terms on the RHS.

$\bullet$ Bound for $\frac{\beta_1\bar\rho\alpha_T}{1-\beta_1} \langle x_T - \hat x_T, m_T \rangle$ in~\eqref{eq: mainmain1}.

Applying Cauchy-Schwarz inequality and using~\Cref{lem: hatx_bd} is enough to bound this term, with $\|m_t \|_\infty \leq G$:
\begin{align}\label{eq: first_est1}
\langle x_T - \hat x_T, m_T \rangle \leq \| x_T - \hat x_T \| \| m_T \| \leq  \hat D \sqrt{d}G.
\end{align}

$\bullet$ Bound for $\frac{\beta_1 \bar\rho}{1-\beta_1} \sum_{t=1}^{T-1} \langle \alpha_t(x_t - \hat{x}_t) - \alpha_{t+1}(x_{t+1} - \hat{x}_{t+1}), m_t \rangle$ in~\eqref{eq: mainmain1}.

We have
\begin{align}
\langle \alpha_t (x_t - \hat{x}_t) - \alpha_{t+1}(x_{t+1} - \hat{x}_{t+1}), m_t \rangle = (\alpha_{t}-\alpha_{t+1}) &\langle x_{t+1} - \hat{x}_{t+1}, m_t \rangle + \alpha_t \langle x_t - x_{t+1}, m_t \rangle \notag \\
 + \alpha_t &\langle \hat x_{t+1} - \hat x_t, m_t\rangle. \label{eq: bd1}
\end{align}
For the first term in~\eqref{eq: bd1}, we use that $\alpha_t \geq
\alpha_{t+1}$,~\Cref{lem: hatx_bd}, Cauchy-Schwarz inequality and $\|m_t \|_{\infty}\leq G$
 to obtain
 \begin{align}
\sum_{t=1}^{T-1} (\alpha_{t}-\alpha_{t+1}) \langle x_{t+1} - \hat{x}_{t+1}, m_t \rangle &\leq  \sum_{t=1}^{T-1} (\alpha_t - \alpha_{t+1}) \hat D\sqrt{d}G \leq \alpha_1 \hat D\sqrt{d}G.\notag
\end{align}
For the second term of~\eqref{eq: bd1}, using nonexpansiveness of weighted projection, we deduce
\begin{align}
\alpha_t \langle x_t - x_{t+1}, m_t \rangle &\leq \alpha_t \| x_t - x_{t+1} \|_{\hat v_t^{1/2}} \| m_t \|_{\hat v_t^{-1/2}} \notag \\
&= \alpha_t \| x_t - P_\mathcal{X}^{\hat v_t^{1/2}}(x_t - \alpha_t \hat v_t^{-1/2}m_t) \|_{\hat v_t^{1/2}} \| m_t \|_{\hat v_t^{-1/2}} \notag \\
&\leq \alpha_t^2 \| m_t \|^2_{\hat v_t^{-1/2}}. \notag
\end{align}
First, summing~\eqref{eq: bd1}, multiplying both sides of the inequality by $\frac{\beta_1\bar\rho}{1-\beta_1}$, and then plugging the last two bounds, we have
\\

\begin{align}
\frac{\beta_1 \bar\rho}{1-\beta_1} \sum_{t=1}^{T-1} &\langle \alpha_t(x_t - \hat x_t) - \alpha_{t+1}(x_{t+1} - \hat x_{t+1}), m_t \rangle \notag \\
&\leq \frac{\beta_1 \bar\rho}{1-\beta_1} \alpha_1  \hat D\sqrt{d}G + \sum_{t=1}^T \frac{\beta_1\bar\rho \alpha_t^2}{1-\beta_1} \| m_t \|^2_{\hat v_t^{-1/2}} + \sum_{t=1}^{T-1}\frac{\beta_1 \bar\rho \alpha_t}{1-\beta_1} \langle \hat x_{t+1} - \hat x_t, m_t \rangle \notag \\
&\leq \frac{\beta_1 \bar\rho}{1-\beta_1} \alpha_1  \hat D\sqrt{d}G + \sum_{t=1}^T\frac{\beta_1\bar\rho \alpha_t^2}{1-\beta_1} \| m_t \|^2_{\hat v_t^{-1/2}} + \sum_{t=1}^T\frac{\bar\rho-\hat\rho}{4} \| \hat x_{t+1} - \hat x_t \|^2_{\hat v_t^{1/2}} \notag \\
&\qquad + \frac{\bar\rho^2}{(\bar\rho-\hat\rho)}\frac{\beta_1^2}{(1-\beta_1)^2} \sum_{t=1}^T\alpha_t^2 \| m_t \|^2_{\hat v_t^{-1/2}},\label{eq: second_main_term}
\end{align}
where we used Young's inequality in the last step.

$\bullet$ Bound for $\sum_{t=1}^T \bar\rho\alpha_t \langle x_t- \hat{x}_t, m_t \rangle$ in~\eqref{eq: mainmain1}.

We proceed as in eq. (3.6) to (3.8) in~\cite{davis2019stochastic}, but
with a tighter bound in the beginning, where we use $x\mapsto f(x) + I_{\mathcal{X}}(x)
+ \frac{\bar\rho}{2} \| x-x_{t+1}\|^2_{\hat v_{t+1}^{1/2}}$ being $\bar\rho -
\hat\rho$ strongly convex w.r.t. $\| \cdot \|_{\hat v_{t+1}^{1/2}}$, with the minimizer $\hat x_{t+1}$
\begin{align}
\phi^{t+1}_{1/\bar \rho}&(x_{t+1}) \leq f(\hat{x}_t) + \frac{\bar \rho}{2} \| \hat{x}_t - x_{t+1} \|^2_{\hat{v}_{t+1}^{1/2}} - \frac{\bar\rho-\hat\rho}{2} \| \hat x_t - \hat x_{t+1} \|^2_{\hat v_{t+1}^{1/2}}\notag \\
&= f(\hat{x}_t) + \frac{\bar \rho}{2} \| \hat{x}_t - x_{t+1} \|^2_{\hat{v}_t^{1/2}} + \frac{\bar\rho}{2} \| \hat{x}_t - x_{t+1} \|^2_{\hat{v}_{t+1}^{1/2} - \hat v_t^{1/2}}- \frac{\bar\rho-\hat\rho}{2} \| \hat x_t - \hat x_{t+1} \|^2_{\hat v_{t+1}^{1/2}}. \label{eq: dd_eq1}
\end{align}
We estimate the second term in the RHS of~\eqref{eq: dd_eq1} by the definition of $x_{t+1}$, then using $\hat x_t \in \mathcal{X}$ and nonexpansiveness of the weighted projection in the weighted norm
\begin{align}
\frac{\bar\rho}{2} \| \hat x_t - x_{t+1} \|^2_{\hat v_t^{1/2}} &= \frac{\bar \rho}{2} \| P_{\mathcal{X}}^{\hat{v}_t^{1/2}}(x_t - \alpha_t \hat v_t^{-1/2} m_t) - \hat{x}_t \|^2_{\hat v_t^{1/2}} \notag \\
&=\frac{\bar \rho}{2} \| P_{\mathcal{X}}^{\hat{v}_t^{1/2}}(x_t - \alpha_t \hat v_t^{-1/2} m_t) - P_{\mathcal{X}}^{\hat v_t^{1/2}}(\hat{x}_t) \|^2_{\hat v_t^{1/2}} \notag \\
&\leq \frac{\bar \rho}{2} \| x_t - \alpha_t \hat v_t^{-1/2} m_t - \hat{x}_t \|^2_{\hat v_t^{1/2}} \notag \\
&=\frac{\bar \rho}{2} \| x_t - \hat{x}_t \|^2_{\hat{v}_t^{1/2}} + \bar \rho \langle \hat{x}_t - x_t, \alpha_t  m_t \rangle + \frac{\bar \rho}{2} \alpha_t^2 \| m_t \|^2_{\hat v_t^{-1/2}}.\notag
\end{align}
We insert this estimate into~\eqref{eq: dd_eq1} and use the definition of $\phi^{t}_{1/\bar\rho}(x_t)$ to obtain
\begin{align}
\phi^{t+1}_{1/\bar \rho}(x_{t+1}) \leq \phi^t_{1/\bar \rho}(x_t) + \bar \rho \alpha_t \langle \hat{x}_t - x_t, m_t \rangle + \frac{\bar \rho}{2} \alpha_t^2 \| m_t \|^2_{\hat v_t^{-1/2}} &+ \frac{\bar\rho}{2}\| \hat{x}_t - x_{t+1} \|^2_{\hat{v}_{t+1}^{1/2} - \hat v_t^{1/2}} \notag \\
&- \frac{\bar\rho-\hat\rho}{2} \| \hat x_t - \hat x_{t+1} \|^2_{\hat v_{t+1}^{1/2}}.\label{eq: dd_end}
\end{align}
We will manipulate the second to last term, by using $\|a+b\|^2\leq2\|a\|^2+2\|b\|^2$, $\hat v_{t+1, i} \geq \hat v_{t, i}$, and~\Cref{lem: hatx_bd}
\begin{align}
\frac{\bar\rho}{2}\| \hat{x}_t - x_{t+1} \|^2_{\hat{v}_{t+1}^{1/2} - \hat v_t^{1/2}} &\leq \bar\rho\| \hat{x}_t - x_{t} \|^2_{\hat{v}_{t+1}^{1/2} - \hat v_t^{1/2}} + \frac{{G}\bar\rho}{\sqrt{\delta}}\| {x}_t - x_{t+1} \|^2_{\hat{v}_{t}^{1/2}}\notag \\
&\leq \bar\rho \hat D^2 \sum_{i=1}^d (\hat v_{t+1, i}^{1/2} - \hat v_{t, i}^{1/2}) + \frac{G\bar\rho}{\sqrt{\delta}}\alpha_t^2 \| m_t \|^2_{\hat{v}_{t}^{-1/2}}.\notag
\end{align}
We use this estimate in~\eqref{eq: dd_end} and sum the inequality to get
\begin{align}
\bar \rho \alpha_t \sum_{t=1}^T \langle x_t - \hat{x}_t, m_t \rangle &\leq \phi^1_{1/\bar\rho}(x_1) - \phi^{T+1}_{1/\bar\rho}(x_{T+1}) + \sum_{t=1}^T \left(\frac{1}{2} + \frac{{G}}{\sqrt{\delta}} \right) \bar\rho\alpha_t^2\| m_t \|^2_{\hat v_t^{-1/2}} \notag \\
&\qquad+ \bar\rho \hat D^2\sum_{i=1}^d \hat v_{T+1, i}^{1/2} - \sum_{t=1}^T \frac{\bar\rho-\hat\rho}{2} \| \hat x_t - \hat x_{t+1} \|^2_{\hat v_{t+1}^{1/2}}.\label{eq: third_main_term}
\end{align}

\textbf{Combining estimates into~\eqref{eq: mainmain1}}. We now plug in~\eqref{eq: first_est1},~\eqref{eq: second_main_term},~\eqref{eq: third_main_term} into~\eqref{eq: mainmain1} and use $\alpha_T \leq \alpha$, $\hat v_{t+1}^{1/2} \geq \hat v_t^{1/2}$ to get
\begin{align}
&\sum_{t=1}^T \bar\rho\alpha_t \langle x_t - \hat{x}_t, g_t \rangle \leq \frac{2\beta_1\bar\rho\alpha}{(1-\beta_1)}\hat D \sqrt{d}G +\phi^1_{1/\bar\rho}(x_1) - \phi^{T+1}_{1/\bar\rho}(x_{T+1}) +  \bar\rho \hat D^2\sum_{i=1}^d \hat v_{T+1, i}^{1/2} \notag \\
&+ \sum_{t=1}^T \left(\frac{1}{2} + \frac{{G}}{\sqrt{\delta}} + \frac{\beta_1}{1-\beta_1} + \frac{\bar\rho}{\bar\rho-\hat\rho}\frac{\beta_1^2}{(1-\beta_1)^2} \right) \bar\rho \alpha_t^2\| m_t \|^2_{\hat v_t^{-1/2}} - \sum_{t=1}^T \frac{\bar\rho - \hat\rho}{4} \| \hat x_t - \hat x_{t+1} \|^2_{\hat v_{t+1}^{1/2}}. \label{eq: lala1}
\end{align}
At this point, due to the coupling between $\hat x_t$, $\hat v_t$, and $g_t$, we cannot directly take expectations, so we will use the estimations of~\Cref{lem: lhs}.
First we sum the result of~\Cref{lem: lhs} which gives
\begin{align}
\sum_{t=1}^T& \mathbb{E}_t \left[ \alpha_t \langle x_t - \hat x_t, g_t \rangle \right] \geq \sum_{t=1}^T \mathbb{E}_t (\bar\rho - \hat\rho) \alpha_t\| x_t - \hat x_t \|^2_{\hat v_t^{1/2}} - (\alpha_0) \sqrt{d}\hat D G  \notag \\
&- \sum_{t=1}^T\frac{\bar\rho-\hat\rho}{4\bar\rho} \mathbb{E}_t \| \hat x_t - \hat x_{t-1} \|^2_{\hat v_{t-1}^{1/2}} - \sum_{t=1}^T\frac{\alpha_{t-1}}{2} \mathbb{E}_t \| m_{t-1} \|^2_{\hat v_{t-1}^{-1/2}} -\sum_{t=1}^T \left( \frac{1}{2} + \frac{\bar\rho}{\bar\rho-\hat\rho} \right)\frac{\alpha_{t-1}^2}{\sqrt{\delta}} \mathbb{E}_t \| g_t \|^2.\notag
\end{align}
We use here the assignments used for convenience: $\alpha_0 = 0$ and $\hat x_{0}=\hat x_1$ and recall that $m_0 = 0$.

We plug this estimation after taking full expectation into~\eqref{eq: lala1} and use $\hat v_{t-1}^{1/2}\leq \hat v_t^{1/2}$ to obtain
\begin{align}
\bar\rho(\bar\rho - \hat\rho )\sum_{t=1}^T \alpha_t \mathbb{E} \| x_t - \hat x_t \|^2_{\hat v_t^{1/2}} &\leq \frac{2\beta_1\bar\rho\alpha}{(1-\beta_1)}\hat D \sqrt{d}G +\phi^1_{1/\bar\rho}(x_1) - \mathbb{E} \phi^{T+1}_{1/\bar\rho}(x_{T+1}) +  \bar\rho \hat D^2\sum_{i=1}^d \mathbb{E} \hat v_{T+1, i}^{1/2} \notag \\
&+ \sum_{t=1}^T \left(1 + \frac{{G}}{\sqrt{\delta}} + \frac{\beta_1}{1-\beta_1} + \frac{\bar\rho}{\bar\rho-\hat\rho}\frac{\beta_1^2}{(1-\beta_1)^2} \right) \bar\rho \alpha_t^2\mathbb{E}\| m_t \|^2_{\hat v_t^{-1/2}} \notag \\
& + \sum_{t=1}^T \left( \frac{1}{2\sqrt{\delta}} +\frac{\bar\rho}{(\bar\rho-\hat\rho)\sqrt{\delta}} \right) \bar\rho \alpha_{t-1}^2 \mathbb{E} \| g_t \|^2.\notag
\end{align}
The only quantities left to estimate are $\sum_{t=1}^T \alpha_{t-1}^2 \| g_t \|^2$ and $\sum_{t=1}^T \alpha_t^2 \| m_t \|^2_{\hat v_t^{-1/2}}$.
Using~\Cref{lem: sum_mt} and $\alpha_0 =0$ shows that both these quantities are bounded by $\mathcal{O}(\log T)$:
\begin{equation}
\sum_{t=1}^T \alpha_t^2 \| m_t \|^2_{\hat v_t^{1/2}} \leq \frac{1-\beta_1}{\sqrt{(1-\beta_2)(1-\gamma)}} dG(1+\log T).\notag
\end{equation}
\begin{equation}
\sum_{t=1}^T \alpha_{t-1}^2 \| g_t\| ^2 = \sum_{t=2}^T \alpha_{t-1}^2 \| g_t\| ^2 \leq dG^2(1+\log T).\notag
\end{equation}
The proof then follows by using~\eqref{eq: mor1}, $f^\star \leq f(x), \forall x\in\mathcal{X}$, picking $\bar\rho = 2\hat\rho$, using $\alpha_t \geq \alpha_T$, and in the end dividing both sides by $T \alpha_T$.
\end{proof}

Before, moving onto the proof of~\Cref{th: th2}, we need a lemma analogous to~\Cref{lem: sum_mt}.
This lemma can be seen as a simplified version of the similar results, for example in~\cite{reddi2018convergence,alacaoglu2020new}.
\begin{lemma}
Let~\Cref{as: as1} hold. Let $\beta_1 < 1$ and $\alpha_t,  v_t$ are set as in~\eqref{eq: adagrad}. Then, we have
\begin{equation}
\sum_{t=1}^T \frac{\alpha_t^2}{v_t} \| m_t \|^2 \leq \alpha d \left( 1+ \log\left( \frac{TG^2}{\delta} +1 \right) \right).\notag
\end{equation}
\end{lemma}
\begin{proof}
We note that $\frac{\alpha_t^2}{v_t} = \frac{\alpha}{\delta + \frac{1}{d} \sum_{j=1}^t \| g_j \|^2}$.
We proceed as~\citep[Lemma 5, 6]{alacaoglu2020new} with the difference of not having diagonal $v_t$:
\begin{align}
\sum_{t=1}^T \frac{\alpha_t^2}{v_t} \| m_t \|^2 &= \sum_{t=1}^T \frac{\alpha_t^2}{v_t} \sum_{i=1}^d (m_{t, i})^2 = \sum_{t=1}^T \frac{\alpha_t^2}{v_t} \sum_{i=1}^d \left(\sum_{j=1}^t (1-\beta_1)\beta_1^{t-j}g_{j, i}\right)^2 \notag \\
&\leq(1-\beta_1)^2 \sum_{t=1}^T \frac{\alpha_t^2}{v_t} \sum_{i=1}^d \left(\sum_{j=1}^t \beta_1^{t-j}\right)\left(\sum_{j=1}^t \beta_1^{t-j}g_{j, i}^2\right)\label{eq: ada_st1}\\
&\leq (1-\beta_1)\alpha \sum_{i=1}^d \sum_{t=1}^T \sum_{j=1}^t \frac{\beta_1^{t-j}g_{j,i}^2}{\delta + \frac{1}{d} \sum_{k=1}^t \| g_k \|^2} \label{eq: ada_st2} \\
&\leq (1-\beta_1)\alpha \sum_{i=1}^d \sum_{t=1}^T \sum_{j=1}^t \frac{\beta_1^{t-j}g_{j,i}^2}{\delta + \frac{1}{d} \sum_{k=1}^j \| g_k \|^2} \label{eq: ada_st3} \\
&= (1-\beta_1)\alpha \sum_{i=1}^d \sum_{j=1}^T \sum_{t=j}^T \frac{\beta_1^{t-j}g_{j,i}^2}{\delta + \frac{1}{d} \sum_{k=1}^j \| g_k \|^2} \label{eq: ada_st4} \\
&\leq \alpha \sum_{i=1}^d \sum_{j=1}^T \frac{g_{j,i}^2}{\delta + \frac{1}{d} \sum_{k=1}^j \| g_k \|^2} \label{eq: ada_st5} \\
&= \alpha\sum_{j=1}^T \frac{\|g_j\|^2}{\delta + \frac{1}{d} \sum_{k=1}^j \| g_k \|^2},\notag
\end{align}
where~\eqref{eq: ada_st1} is by Cauchy-Schwarz inequality,~\eqref{eq: ada_st2} is by summing a geometric series,~\eqref{eq: ada_st3} is by $j\leq t$,~\eqref{eq: ada_st4} is by changing the order of summation,~\eqref{eq: ada_st5} is by summing a geometric series and the last step is by changing the order of summation.

Now we can apply a standard inequality, for nonnegative numbers $a_i, \forall i$ and $\delta >0$~\citep[Lemma A.3]{levy2018online}
\begin{equation}
\sum_{j=1}^T \frac{a_j}{\delta + \sum_{k=1}^j a_j} \leq 1 + \log\left( \frac{\sum_{j=1}^T a_j}{\delta} +1 \right)\notag
\end{equation}
to conclude.
\end{proof}
\begin{reptheorem}{th: th2}
Let~\Cref{as: as1} hold. Then, for the method sketched in~\eqref{eq: adagrad}, with $\beta_1 < 1$, $\alpha_t =\frac{\alpha}{\sqrt{t}}$ it holds
\begin{align}
\mathbb{E} \| \nabla \phi_{1/2\rho}&(x_{t^\ast})\|^2 \leq\frac{2G}{\alpha\sqrt{T} } \bigg[ C_1 + \left(1+\log\left(\frac{TG^2}{\delta}+1\right)\right)C_2 \bigg],\notag
\end{align}
where $C_1 = \phi_{1/2\rho}(x_1) - f^\star + 2\rho\left( \frac{2\beta_1}{1-\beta_1} +1 \right)\frac{\alpha \hat D \sqrt{d}G}{\sqrt{\delta}}$,
$C_2 = 2\rho\alpha d\left( \frac{1}{2} + \frac{\beta_1}{1-\beta_1} + \frac{2\beta_1^2}{(1-\beta_1)^2} \right)$, and $\hat D = \frac{2\sqrt{d}G}{\sqrt{\rho}}$.
\end{reptheorem}
\begin{proof}
This proof will be midway between the proof we have presented for~\Cref{th: th1} and the proof from~\cite{davis2019stochastic} for standard SGD.

We recall the definitions
\begin{align}
&\hat x_t = \argmin_{x\in\mathcal{X}} f(x) + \frac{\bar\rho}{2} \| x - x_t \|^2,\notag\\
&\phi_{1/\bar\rho}(x_t) = \min_{x\in\mathcal{X}} f(x) + \frac{\bar\rho}{2} \| x-x_t \|^2, \notag\\
&\frac{\alpha_t}{\sqrt{v_t}} = \frac{\alpha}{\sqrt{\delta +\frac{1}{d} \sum_{j=1}^t \| g_j \|^2}}.\notag
\end{align}
Same as~\Cref{th: th1}, we sum the result of Lemma~\ref{lem: lem1} to get
\begin{align}
\sum_{t=1}^T \langle A_t, g_t \rangle = \frac{\beta_1}{1-\beta_1} \langle A_T, m_T \rangle + \sum_{t=1}^T \langle A_t, m_t \rangle + \frac{\beta_1}{1-\beta_1}\sum_{t=1}^{T-1} \langle A_t - A_{t+1}, m_t \rangle.\label{eq: key1_adagrad}
\end{align}
We now let $A_t = \bar\rho\frac{\alpha_t}{\sqrt{v_t}}(x_t - \hat{x}_t)$ and~\eqref{eq: key1_adagrad} becomes
\begin{align}
\sum_{t=1}^T \bar\rho\frac{\alpha_t}{\sqrt{v_t}} \langle x_t - \hat{x}_t, g_t \rangle &\leq \frac{\beta_1\bar\rho\alpha_T}{\sqrt{v_T}(1-\beta_1)} \langle x_T - \hat{x}_T, m_T \rangle + \sum_{t=1}^T \bar\rho\frac{\alpha_t}{\sqrt{v_t}} \langle x_t- \hat{x}_t, m_t \rangle \notag \\
&+ \frac{\beta_1 \bar\rho}{1-\beta_1} \sum_{t=1}^{T-1} \langle \frac{\alpha_t}{\sqrt{v_t}}(x_t - \hat{x}_t) - \frac{\alpha_{t+1}}{\sqrt{v_{t+1}}}(x_{t+1} - \hat{x}_{t+1}), m_t \rangle.\label{eq: mainmain1_adagrad}
\end{align}
$\bullet$ Bound for $\frac{\beta_1 \bar\rho}{1-\beta_1} \sum_{t=1}^{T-1} \langle \frac{\alpha_t}{\sqrt{v_t}}(x_t - \hat{x}_t) - \frac{\alpha_{t+1}}{\sqrt{v_{t+1}}}(x_{t+1} - \hat{x}_{t+1}), m_t \rangle$ in~\eqref{eq: mainmain1_adagrad}

We deduce similar to~\eqref{eq: bd1}
\begin{align}
\langle \frac{\alpha_t}{\sqrt{v_t}} (x_t - \hat{x}_t) - \frac{\alpha_{t+1}}{\sqrt{v_{t+1}}}(x_{t+1} - \hat{x}_{t+1}), m_t \rangle &= \left(\frac{\alpha_{t}}{\sqrt{v_{t}}}-\frac{\alpha_{t+1}}{\sqrt{v_{t+1}}}\right) \langle x_{t+1} - \hat{x}_{t+1}, m_t \rangle \notag \\
&+ \frac{\alpha_t}{\sqrt{v_t}} \langle x_t - x_{t+1}, m_t \rangle + \frac{\alpha_t}{\sqrt{v_t}} \langle \hat x_{t+1} - \hat x_t, m_t\rangle.\notag
\end{align}

We note that since $\frac{\alpha_t}{\sqrt{v_t}}$ is decreasing,
\begin{align}
\sum_{t=1}^{T-1} \left(\frac{\alpha_{t}}{\sqrt{v_{t}}}-\frac{\alpha_{t+1}}{\sqrt{v_{t+1}}}\right) \langle x_{t+1} - \hat{x}_{t+1}, m_t \rangle &\leq \sum_{t=1}^{T-1} (\frac{\alpha_t}{\sqrt{v_t}} - \frac{\alpha_{t+1}}{\sqrt{v_{t+1}}}) \hat D \sqrt{d} G \notag \\
&\leq \frac{\alpha_1}{\sqrt{v_1}} \hat D \sqrt{d} G.\notag
\end{align}
Next, we use Cauchy-Schwarz inequality, definition of $x_{t+1}$ and nonexpansiveness
\begin{equation}
\frac{\alpha_t}{\sqrt{v_t}} \langle x_t - x_{t+1}, m_t \rangle \leq \frac{\alpha_t}{\sqrt{v_t}} \| x_t - P_{\mathcal{X}} (x_t - \frac{\alpha_t}{\sqrt{v_t}} m_t ) \| \|m_t\| \leq \frac{\alpha_t^2}{{v_t}}\| m_t \|^2.\notag
\end{equation}
We use Young's inequality to get
\begin{equation}
\frac{\alpha_t}{\sqrt{v_t}} \langle \hat x_{t+1} - \hat x_t, m_t \rangle \leq \frac{(\bar\rho - \rho)(1-\beta_1)}{4\bar\rho\beta_1} \| \hat x_{t+1} - \hat x_t \|^2 + \frac{\alpha_t^2\bar\rho\beta_1}{v_t(\bar\rho-\rho)(1-\beta_1)} \| m_t \|^2.\notag
\end{equation}
Collecting all the bounds in this part gives
\begin{multline}
\frac{\beta_1 \bar\rho}{1-\beta_1} \sum_{t=1}^{T-1} \langle \frac{\alpha_t}{\sqrt{v_t}}(x_t - \hat{x}_t) - \frac{\alpha_{t+1}}{\sqrt{v_{t+1}}}(x_{t+1} - \hat{x}_{t+1}), m_t \rangle \leq \frac{\bar\rho\beta_1\alpha_1}{(1-\beta_1)\sqrt{v_1}} \hat D\sqrt{d}G \\
+ \frac{\bar\rho - \rho}{4} \sum_{t=1}^T \| \hat x_{t+1} - \hat x_t \|^2 + \sum_{t=1}^T \left( \frac{\bar\rho\beta_1}{1-\beta_1} + \frac{\bar\rho^2\beta_1^2}{(\bar\rho - \rho)(1-\beta_1)^2} \right)\frac{\alpha_t^2}{v_t} \| m_t \|^2.\label{eq: ada_bd1}
\end{multline}

$\bullet$ Bound for $\sum_{t=1}^T \bar\rho\frac{\alpha_t}{\sqrt{v_t}} \langle x_t- \hat{x}_t, m_t \rangle$ in~\eqref{eq: mainmain1_adagrad}

Since $x \mapsto f(x) +I_{\mathcal{X}}(x) + \frac{\bar\rho}{2}
\|x-x_{t+1} \|^2$ is $(\bar\rho-\rho)$-strongly convex with the minimizer, $\hat x_{t+1}$
\begin{equation}
\phi_{1/\bar\rho}(x_{t+1}) \leq f(\hat x_t) + \frac{\bar\rho}{2} \| \hat x_t -x_{t+1} \|^2 - \frac{\bar\rho - \rho}{2} \| \hat x_t - \hat x_{t+1} \|^2.\label{eq: moreau_adagrad}
\end{equation}
By using $\hat x_t \in \mathcal{X}$
\begin{align}
\frac{\bar\rho}{2}\| x_{t+1} - \hat x_t \|^2 &= \frac{\bar\rho}{2}\| P_{\mathcal{X}}(x_t - \frac{\alpha_t}{\sqrt{v_t}} m_t) - P_{\mathcal{X}}(\hat x_t) \|^2 \leq \frac{\bar\rho}{2} \| x_t - \frac{\alpha_t}{\sqrt{v_t}} m_t- \hat x_t \|^2 \notag \\
&= \frac{\bar\rho}{2}\| x_t - \hat x_t \|^2 - \frac{\bar\rho\alpha_t}{\sqrt{v_t}} \langle x_t - \hat x_t, m_t \rangle + \frac{\bar\rho}{2}\frac{\alpha_t^2}{v_t} \| m_t \|^2.\notag
\end{align}
Then,~\eqref{eq: moreau_adagrad} becomes
\begin{align}
\phi_{1/\bar\rho}(x_{t+1}) \leq \phi_{1/\bar\rho}(x_t) -\frac{\bar\rho\alpha_t}{\sqrt{v_t}} \langle x_t - \hat x_t, m_t \rangle + \frac{\bar\rho\alpha_t^2}{2v_t} \| m_t \|^2 - \frac{\bar\rho-\rho}{2} \| \hat x_{t+1} - \hat x_t \|^2.\notag
\end{align}
Summing this inequality gives
\begin{equation}
\sum_{t=1}^T \frac{\bar\rho\alpha_t}{\sqrt{v_t}} \langle x_t - \hat x_t, m_t \rangle \leq \phi_{1/\bar\rho}(x_1) - \phi_{1/\bar\rho}(x_{T+1}) + \bar\rho\sum_{t=1}^T \frac{\alpha_t^2}{2v_t} \| m_t \|^2 - \frac{\bar\rho-\rho}{2} \sum_{t=1}^T \| \hat x_{t+1} - \hat x_t \|^2.\label{eq: adagrad_moreau_result}
\end{equation}
We now collect~\eqref{eq: ada_bd1} and~\eqref{eq: adagrad_moreau_result} into~\eqref{eq: mainmain1_adagrad}
\begin{multline}
\sum_{t=1}^T \bar\rho\frac{\alpha_t}{\sqrt{v_t}} \langle x_t - \hat x_t, g_t \rangle \leq \frac{\bar\rho \beta_1 \alpha_T}{\sqrt{v_T}(1-\beta_1)} \hat D\sqrt{d} G + \phi_{1/\bar\rho}(x_1) - \phi_{1/\bar\rho}(x_{T+1}) \\
+ \bar\rho \sum_{t=1}^T \left( \frac{1}{2} + \frac{\beta_1}{1-\beta_1} + \frac{\bar\rho\beta_1^2}{(\bar\rho-\rho)(1-\beta_1)^2} \right) \frac{\alpha_t^2}{v_t} \| m_t \|^2 + \frac{\bar\rho\beta_1\alpha_1}{\sqrt{v_1}(1-\beta_1)} \hat D \sqrt{d} G.\label{eq: ada_bd3}
\end{multline}
Due to coupling of $v_t$ and $g_t$, we estimate LHS as
\begin{align}
\mathbb{E}_t \bar\rho\frac{\alpha_t}{\sqrt{v_t}} \langle x_t - \hat x_t, g_t \rangle &= \bar\rho\frac{\alpha_{t-1}}{\sqrt{v_{t-1}}} \langle x_t - \hat x_t, \mathbb{E}_t g_t \rangle + \bar\rho\mathbb{E}_t\left( \frac{\alpha_t}{\sqrt{v_t}} - \frac{\alpha_{t-1}}{\sqrt{v_{t-1}}}\right)\langle x_t - \hat x_t, g_t \rangle \notag \\
&\geq \bar\rho\frac{\alpha_{t-1}}{\sqrt{v_{t-1}}} (\bar\rho-\rho)\mathbb{E}_t\| x_t - \hat x_t \|^2 - \mathbb{E}_t \bar\rho \left( \frac{\alpha_{t-1}}{\sqrt{v_{t-1}}} -\frac{\alpha_t}{\sqrt{v_t}} \right) \hat D \sqrt{d}G,\notag
\end{align}
where in the last line we used the estimation~\eqref{eq: dd_std_lhs_est} without weighted norms.
It is clear that above inequality holds for any $t>1$; in order to
have it for $t=1$, we have to define $\a_0$, which we can choose
arbitrarily. For convenience, we set $\frac{\alpha_0}{\sqrt{v_0}} =
\frac{\alpha_1}{\sqrt{v_1}}$. Then
we take expectation of~\eqref{eq: ada_bd3}, use $\frac{\alpha_T}{\sqrt{v_T}} \leq \frac{\alpha_1}{\sqrt{v_1}} \leq \frac{\alpha}{\sqrt{\delta}}$ and plug in the last inequality to get
\begin{multline}
\sum_{t=1}^T \bar\rho(\bar\rho-\rho) \frac{\alpha_{t-1}}{\sqrt{v_{t-1}}}  \mathbb{E} \| x_t - \hat x_t \|^2 \leq \phi_{1/\bar\rho}(x_1) - \mathbb{E}\phi_{1/\bar\rho}(x_{T+1}) + \bar\rho\frac{2\alpha}{\sqrt{\delta}}\frac{\beta_1}{1-\beta_1} \hat D \sqrt{d} G \\
+ \bar\rho \sum_{t=1}^T \left( \frac{1}{2} + \frac{\beta_1}{1-\beta_1} + \frac{\bar\rho\beta_1^2}{(\bar\rho-\rho)(1-\beta_1)^2} \right) \mathbb{E}\frac{\alpha_t^2}{v_t} \| m_t \|^2 + \bar\rho \frac{\alpha}{\sqrt{\delta}} \hat D\sqrt{d} G.\notag
\end{multline}
We now note $\frac{\alpha_{t-1}}{\sqrt{v_t-1}} \geq \frac{\alpha_{T-1}}{\sqrt{v_{T-1}}}\geq \frac{\alpha}{\sqrt{T}G}$.
We also use $\bar\rho^2\| \hat x_t - x_t \| = \| \nabla \phi_{1/\bar\rho}(x) \|^2$.

For $\hat D$, we use~\Cref{lem: hatx_bd} without the metric to obtain
\begin{equation}
\| \hat x_t - x_t \|^2\leq \hat D^2 = \frac{4dG^2}{\bar\rho-\rho}.\notag
\end{equation}
We select $\bar\rho=2\rho$ and collect the bounds to complete the proof.
\end{proof}

\section{Relation between gradient mapping and Moreau envelope} \label{sec: grad_map_stuff_moreau}
We show how to determine the constant for the inequality
$\|\mathcal{G}_{1/\bar\rho}(x_t) \| \leq C_{\mathrm{g, m}} \| \nabla \phi^t_{1/\bar\rho}(x_t) \|_{\hat v_t^{-1/2}}$, by following arguments similar to~\citep[Theorem 3.5]{drusvyatskiy2018error}.

We start with the definitions
\begin{align*}
  \phi(x) &= f(x) + r(x) \coloneqq f(x) + I_\mathcal{X}(x)\\
  \mathcal{G}_{\lambda}(x) &=  \frac{\hat v_t^{1/4}}{\lambda} \left( x -
  P^{\hat v_t^{1/2}}_\mathcal{X}(x - \lambda\hat v_t^{-1/2} \nabla
  f(x)) \right)\\
\hat{x}_t &= \prox^{\hat{v}_t^{1/2}}_{\phi/\bar{\rho}}(x_t) =
            \argmin_y\left\{ \phi(y) + \frac{\bar{\rho}}{2} \| y-x_t
            \|^2_{\hat v_t^{1/2}}\right\}
\end{align*}
Let us use the notation   $z \coloneqq \nabla \phi^t_{1/\bar\rho}(x_t)
= \bar\rho \hat v_t^{1/2} (x_t - \hat x_t)$ and
$\alpha \coloneqq \bar\rho^{-1}\hat v_t^{-1/2}$.
As $\hat x_t = (I + \bar\rho^{-1}\hat v_t^{-1/2}\partial \phi)^{-1}(x_t)$, we have
\begin{align}
&\qquad\qquad z = \bar\rho\hat v_t^{1/2}(x_t - \hat x_t) \quad\iff \quad\alpha z = x_t  - (I + \alpha \partial \phi)^{-1}(x_t) \notag \\
&\iff\quad x_t \in (I + \alpha \partial \phi)(x_t - \alpha z) \notag \\
&\iff \quad x_t \in (I + \alpha \partial r)(x_t - \alpha z) + \alpha \nabla f(x_t - \alpha z) + \alpha \nabla f(x_t) - \alpha \nabla f(x_t).\notag
\end{align}
Let $w = \alpha \nabla f(x_t - \alpha z) - \alpha \nabla f(x_t)$. Then
\begin{align}
x_t - \alpha\nabla f(x_t) - w \in (I + \alpha\partial r)(x_t - \alpha z) \iff x_t - (I+\alpha \partial r)^{-1}(x_t - \alpha\nabla f(x_t) - w) = \alpha z.\notag
\end{align}
We now plug in the value of $\alpha = \bar\rho^{-1} \hat v_t^{-1/2}$
\begin{align}
\| \bar\rho \hat v_t^{1/4}(x_t - \prox_{r/\bar\rho}^{\hat v_t^{1/2}}(x_t - \bar\rho^{-1}\hat v_t^{-1/2}\nabla f(x_t) - w))) \| = \| \hat v_t^{-1/4} z \|.\label{eq: grad_map_stuff}
\end{align}
By the triangle inequality and nonexpansiveness,
we have   that
\begin{align*}
\text{LHS} &\geq \|\bar\rho\hat v_t^{1/4}(x_t - \prox_{r/\bar\rho}^{\hat v_t^{1/2}}(x_t - \bar\rho^{-1}\hat v_t^{-1/2}\nabla f(x_t))) \| \\
&\qquad- \|\bar\rho\hat v_t^{1/4}(\prox_{r/\bar\rho}^{\hat v_t^{1/2}}(x_t - \bar\rho^{-1}\hat v_t^{-1/2}\nabla f(x_t) - w) - \prox_{r/\bar\rho}^{\hat v_t^{1/2}}(x_t - \bar\rho^{-1}\hat v_t^{-1/2}\nabla f(x_t))) \| \\ &
\geq \| \mathcal{G}_{1/\bar\rho}(x_t) \| - \|\bar\rho w \|_{\hat v_t^{1/2}}.
\end{align*}
Thus, we deduce from~\eqref{eq: grad_map_stuff} that
\begin{equation}
\| \mathcal{G}_{1/\bar\rho}(x_t) \| \leq \| \nabla \phi^t_{1/\bar\rho}(x_t) \|_{\hat v_t^{-1/2}} +  \bar\rho \| w \|_{\hat v_t^{1/2}}.\notag
\end{equation}
We lastly estimate $\| w\|_{\hat v_t^{1/2}}$ using $L$-smoothness of $f$. Let us denote by $\hat L$ the smoothness constant of $f$ w.r.t.\ norm $\| \cdot \|_{\hat v_t^{1/2}}$:
\begin{equation}
\| \nabla f(x) - \nabla f(y) \|_{\hat v_t^{-1/2}} \leq \hat L \| x-y \|_{\hat v_t^{1/2}}.\notag
\end{equation}
Then
\begin{align}
\bar\rho \| w \|_{\hat v_t^{1/2}} &= \| \nabla f(x_t - \bar\rho^{-1} \hat v_t^{-1/2} z) - \nabla f(x_t) \|_{\hat v_t^{-1/2}} \leq \hat L \bar\rho^{-1} \| \hat v_t^{-1/2}z \|_{\hat v_t^{1/2}} \notag \\
&= \hat L \bar\rho^{-1} \| z \|_{\hat v_t^{-1/2}} = \hat L \bar\rho^{-1} \| \nabla \phi^{t}_{1/\bar\rho}(x_t) \|_{\hat v_t^{-1/2}}.\notag
\end{align}
Recall that in our main theorem we have chosen $\bar\rho = 2\hat\rho$ where $\hat\rho$ was the weak convexity constant of $f$ w.r.t. norm $\| \cdot \|_{\hat v_t^{1/2}}$. Similarly, here we have a constant depending on $\hat\rho^{-1}\hat L$, where $\hat L$ is the Lipschitz constant of $f$ on the weighted norm.

\end{document}